\documentclass[preprint,11pt]{elsarticle}

\makeatletter
\def\ps@pprintTitle{%
  \let\@oddhead\@empty
  \let\@evenhead\@empty
  \def\@oddfoot{\reset@font\hfil\thepage\hfil}
  \let\@evenfoot\@oddfoot
}
\makeatother

\usepackage[utf8]{inputenc}
\usepackage{mathtools}
\usepackage{amsfonts}
\usepackage{amsmath}
\usepackage{amsthm}
\usepackage{booktabs}
\usepackage{multicol}
\usepackage[left=3cm,right=3cm,bottom=3cm,top=3cm]{geometry}
\usepackage{xcolor}
\usepackage{algorithm2e}
\usepackage{comment}
\usepackage{url}

\SetKwComment{Comment}{/* }{ */}
\SetKwInput{KwIn}{Input}

\allowdisplaybreaks

\newtheorem{definition}{Definition}
\newtheorem{theorem}{Theorem}

\newtheorem{lemma}{Lemma}

\newtheorem{remark}{Remark}

\DeclarePairedDelimiter{\abs}{\lvert}{\rvert}
\DeclarePairedDelimiter{\norm}{\lVert}{\rVert}
\DeclarePairedDelimiterX{\infdivx}[2]{(}{)}{#1\;\delimsize\|\;#2}

\newcommand{\veccol}[1]{\left(\begin{smallmatrix}#1\end{smallmatrix}\right)}
\newcommand{\eqdef}{\stackrel{\mathclap{\normalfont\tiny\mbox{def}}}{=}}
\newcommand{\KL}{KL\infdivx}

\newcommand*{\R}{\mathbb{R}}

\journal{Neural Networks}

\begin{document}

\begin{frontmatter}

\title{Learning Domain Invariant Representations by Joint Wasserstein Distance Minimization}
  
\author[1,3]{L\'eo~And\'eol}
\ead{leo@andeol.eu}
\author[2,4]{Yusei~Kawakami}
\ead{kawakami-yusei@fujitsu.com}
\author[4,5]{Yuichiro~Wada}
\ead{wada.yuichiro@fujitsu.com}
\author[2,5]{Takafumi~Kanamori\corref{cor1}}
\ead{kanamori@c.titech.ac.jp}
\author[1,3,6,7,8]{Klaus-Robert~M\"uller\corref{cor1}}
\ead{klaus-robert.mueller@tu-berlin.de}
\author[9,1,3]{Gr\'egoire~Montavon\corref{cor1}}
\ead{gregoire.montavon@fu-berlin.de}

\cortext[cor1]{Corresponding authors}
\address[1]{Machine Learning group, Technische Universität Berlin, 10587 Berlin, Germany}
\address[2]{Tokyo Institute of Technology, Tokyo, Japan}
\address[3]{Berlin Institute for the Foundations of Learning and Data -- BIFOLD, 10587 Berlin, Germany}
\address[4]{Fujitsu Laboratories Ltd., Japan}
\address[5]{RIKEN AIP, Japan}
\address[6]{Max Planck Institute for Informatics, Stuhlsatzenhausweg 4, 66123 Saarbr\"ucken, Germany}
\address[7]{Department of Artificial Intelligence, Korea University, Seoul 136-713, South Korea}
\address[8]{Google Deepmind, Berlin, Germany}
\address[9]{Department of Mathematics and Computer Science, Freie Universit\"at Berlin, 14195 Berlin, Germany\vspace{-.5cm}}

\begin{abstract}
Domain shifts in the training data are common in practical applications of machine learning; they occur for instance when the data is coming from different sources. Ideally, a ML model should work well independently of these shifts, for example, by learning a domain-invariant representation. However, common ML losses do not give strong guarantees on how consistently the ML model performs for different domains, in particular, whether the model performs well on a domain at the expense of its performance on another domain. In this paper, we build new theoretical foundations for this problem, by contributing a set of mathematical relations between classical losses for supervised ML and the Wasserstein distance in joint space (i.e.\ representation and output space). We show that classification or regression losses, when combined with a GAN-type discriminator between domains, form an upper-bound to the true Wasserstein distance between domains. This implies a more invariant representation and also more stable prediction performance across domains. Theoretical results are corroborated empirically on several image datasets. Our proposed approach systematically produces the highest minimum classification accuracy across domains, and the most invariant representation.
\end{abstract}


\end{frontmatter}

\section{Introduction}

Learning from data that originates from different provenances representing the same physical observations occurs rather commonly, but it is nevertheless a highly challenging endeavor. 
These multiple data sources may e.g.\ originate from different users,  acquisition devices, geographical locations, they may encompass batch effects in biology, or they may come from the same measurement devices that each are calibrated differently.    
Because the source of the data itself is typically not task-relevant, a learned model is therefore required to be {\em invariant across domains}. A valid strategy for achieving this is to learn an invariant intermediate representation (illustrated in Figure \ref{fig:intro}). Furthermore, in certain applications, privacy requirements such as anonymity dictate that the source should not be recoverable from the representation. Hence, building a domain invariant representation can also be a desideratum {\em by itself}.

Domain invariance, in some contexts referred to as subpopulation shift \cite{koh2020wilds} or distributional shifts \cite{DBLP:journals/corr/AmodeiOSCSM16,goel2020model}, can be contrasted to two related and well-researched areas that are \textit{domain adaptation} (DA) \cite{shimodaira2000improving, sugiyama2007covariate} and \textit{domain generalization} (DG) \cite{dou2019domain, zhou2020domain}. Domain adaptation is mainly concerned with the model performance on the (unlabeled) target domain, often at the expense of incurring more errors on the (labeled) source domain. Domain generalization, on the other hand, aims to build a ML model that generalizes across \textit{all} domains, including unseen ones. This generality imposes additional constraints on the solution, that can hamper the careful enforcement of invariance w.r.t.\ the domains at hand. In comparison, \textit{domain invariance} (DI), our focus in this paper, considers that the ML model is trained and applied on a finite and given set of domains, and each domain is treated equally. The objective is to learn a model whose performance is well-balanced over the multiple given domains. The differences are highlighted graphically and with equations in Figure \ref{fig:comparison-tasks}. Hence, we address a singular and important problem, which has so far received little attention, especially in the context of deep learning models.

In order to address domain invariance, we consider in the present work the \textit{Wasserstein distance} \citep{peyre2019computational,villani2008optimal} as it characterizes the weak convergence of measures and displays several advantages over e.g.\ the more common Kullback-Leibler divergence, as discussed in \citep{arjovsky2017wasserstein,DBLP:conf/nips/MontavonMC16}.
We contribute several bounds relating the Wasserstein distance between the joint distributions of two or more domains, and the objective function of practical supervised neural networks.  This theoretical basis supports the rigorous learning of domain-invariant classifiers through the incorporation of a GAN-type discriminator between domains (or domain critic) as an auxiliary task\footnote{Anecdotally, the use of a domain critic makes our method relate to works on domain adaptation such as DANN \cite{ganin2016domain} and WDGRL \cite{shen2017wasserstein}.}. With the proposed theoretical grounding, one can show that (1) the Wasserstein distance between the different domains is systematically reduced as an effect of training, and (2) the prediction performance gap between domains is also reduced as a result.

\begin{figure}[t!]
    \begin{minipage}[t]{.32\textwidth}
    \centering {\footnotesize \sffamily domain adaptation}\\
    \includegraphics[width=\textwidth]{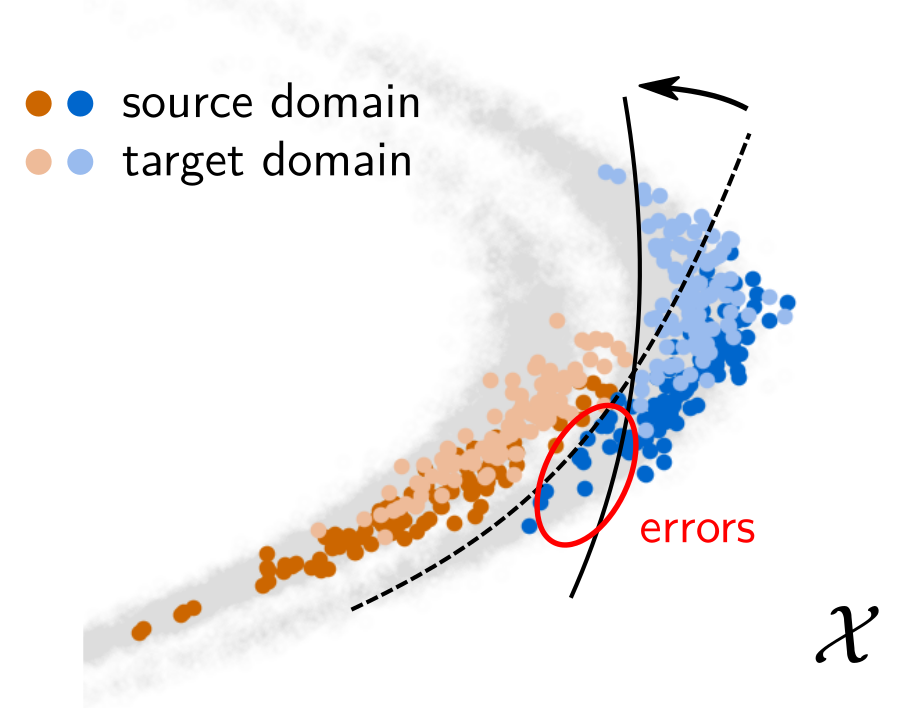}
    $$\mathbb{E}_{(x,y)\sim\mathcal{P}_\text{target}} \mathcal{L}(f(x), y)$$
    \end{minipage} \hfill \vrule \hfill
    \begin{minipage}[t]{.32\textwidth}
    \centering {\footnotesize \sffamily domain invariance (this paper)}\\
    \includegraphics[width=\textwidth]{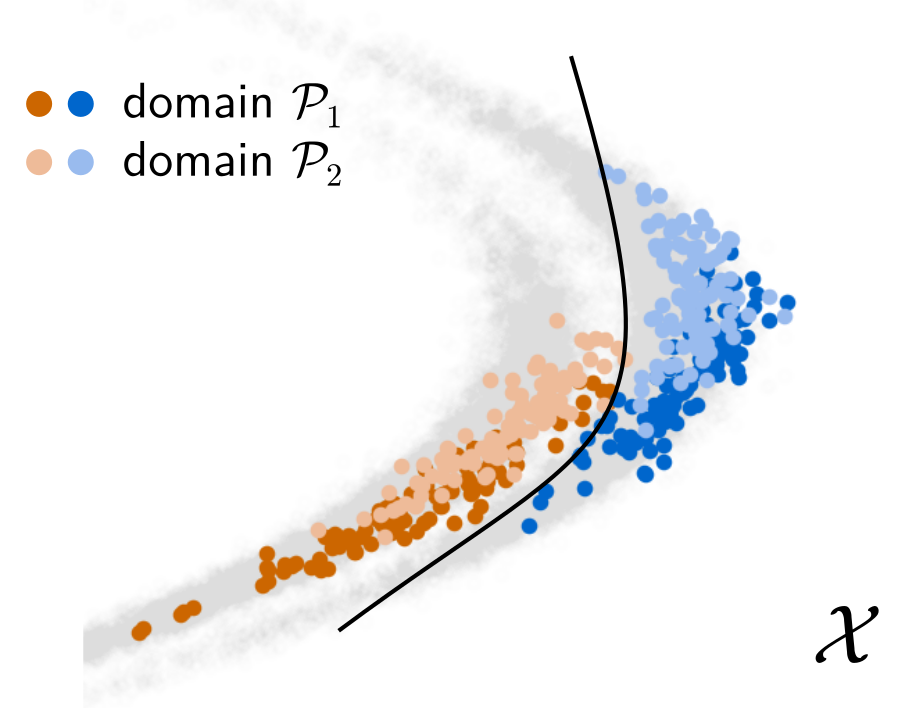}
    $$\frac{1}{n_{d}}\sum\limits_{i = 1}^{n_{d}}\mathbb{E}_{(x,y)\sim\mathcal{P}_i} \mathcal{L}(f(x), y)$$
    \end{minipage} \hfill \vrule \hfill
    \begin{minipage}[t]{.32\textwidth}
    \centering {\footnotesize \sffamily domain generalization}\\
    \includegraphics[width=\textwidth]{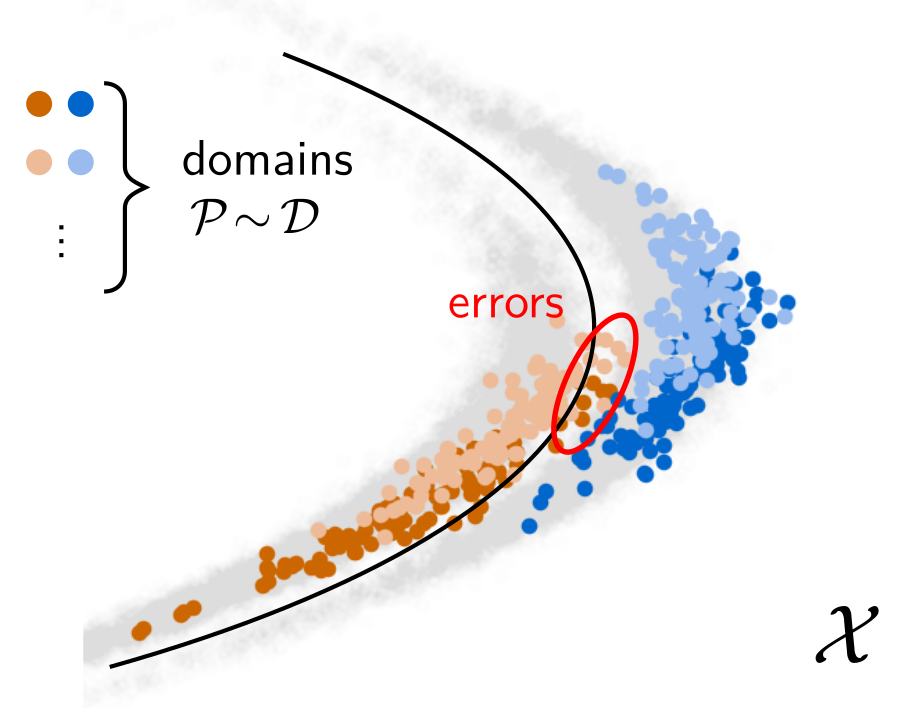}
    $$\mathbb{E}_{\mathcal{P}\sim\mathcal{D}}\mathbb{E}_{(x,y)\sim\mathcal{P}} \mathcal{L}(f(x), y)$$
    \end{minipage}
    \caption{Visual overview of the differences between domain adaptation, domain invariance and domain generalization in the context of classification. $\mathcal{X}$ denotes the input domain, and $\mathcal{P}$ denotes the various probability distributions. \textit{Domain adaptation} learns a classifier that matches the target domain ($\mathcal{P}_\mathrm{target}$) using information from the source domain, irrespective of its performance on it (errors as circled in red). \textit{Domain invariance} treats each of the $n_d$ domains equally and aims to build domain invariant representations and therefore a predictor that works equivalently well on each of them. \textit{Domain generalization} addresses the more complex task of building a classifier that performs well on any domain drawn from some distribution $\mathcal{D}$ (including unseen ones, here depicted in gray). This is done potentially at the cost of giving up some accuracy on the few given domains (errors on the two domains of interest are circled in red).}
    \label{fig:comparison-tasks}
\end{figure}


Furthermore, a significant part of the novelty of our work lies in contributing a formalism, which makes our theory applicable to \textit{partially labeled} distributions. This allows us in particular to cover both supervised and semi-supervised learning scenarios. While a few other works also addressed the scenario where domains are partially labeled, they focus on the related but distinct problems of domain adaptation \citep{cheng2014semi,lopez2013semi,he2020classification} and domain generalization \citep{sharifi2020domain}.

Our proposed approach is tested empirically on three domain invariance benchmarks: MNIST vs.\ SVHN, and the multi-domain Office-Caltech and PACS datasets. Results confirm our theoretical analysis, in particular, we find that our approach yields \textit{highly invariant} representations, and that the latter support predictions that are accurate on \textit{all} domains, including the most difficult ones. Lastly, we inspect the learned invariant representation using UMAP embeddings \cite{mcinnes2018umap} and `explainable AI' (cf.\ \cite{DBLP:conf/kdd/Ribeiro0G16,samek2021}). This allows us to visually highlight how the data distributions associated to each domain merge into a single distribution under the effect of the training objective. It also allows us to explore which input features are used to map the data into the desired invariant representation \cite{liu2019transferable}. Interestingly, we find that recognizing and exploiting \textit{domain-specific} features remains in fact an integral part of the neural network strategy to arrive ultimately at the desired invariant representation.

\begin{figure}[t!]
    \centering
    \includegraphics[width=\textwidth]{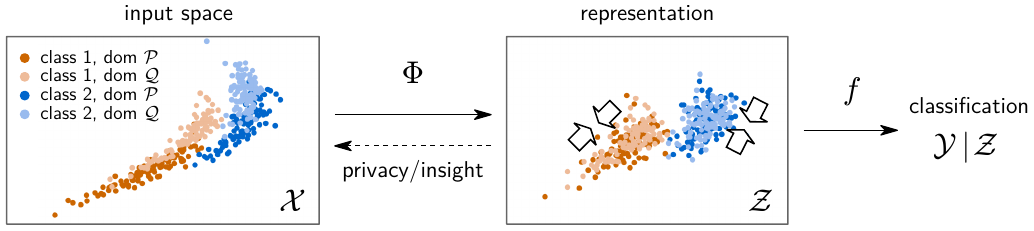}
    \caption{Illustration of the problem of domain invariance in the case of classification. We would like to learn a function $\Phi$ that maps the data to a representation where the domains cannot be differentiated, and from which a domain-invariant classifier $f$ can be built. The invariant representation induced by this model can serve further purposes such as domain privacy or extraction of domain-related insights. $\mathcal{X}$, $\mathcal{Z}$, $\mathcal{Y}$ correspond to the input, representation and target (label) space respectively.}
    \label{fig:intro}
\end{figure}

\section{Related work}

Significant research in machine learning and statistics has been dedicated to the question of distributional shifts (between training and/or test distributions) \cite{DBLP:journals/corr/AmodeiOSCSM16, delage2010distributionally}. This has resulted in a variety of machine learning formulations that can be broadly categorized into domain adaptation, domain generalization, and domain invariance.

\subsection{Domain Adaptation} 

Domain Adaptation \cite{ben2010theory} has been studied due to the fact that in real-world situations, when the source and target distributions differ, for instance by a covariate shift \cite{shimodaira2000improving, sugiyama2007covariate}, models trained on the source distribution perform significantly worse on the target. Domain adaptation has two major well-studied settings: Unsupervised Domain Adaptation and Semi-Supervised Domain Adaptation.

Unsupervised Domain Adaptation considers the situation where the source domain has labels but the target domain has not (cf.\ \cite{zhang2022transfer} for a review). Using the theoretical framework of \cite{ben2010theory} where the target error is upper-bounded by the error on the source domain, the divergence between marginal distribution of the two domains and a constant term, \cite{ganin2016domain, shen2017wasserstein} propose a domain adaptation technique based on the minimization of this upper-bound. The joint distribution optimal transportation (JDOT) method of \cite{courty2017joint} is similar to \cite{shen2017wasserstein}, but it aims to minimize the Wasserstein distance between \textit{joint} distributions. Note that \cite{courty2017joint} does not use adversarial learning and instead solves the primal form of the optimal transport problem, and relies on a single-domain classifier to learn on the target domain with transported source-domain labels. The approach of \cite{ganin2016domain} has been extended to a number of subsequent domain adaptation methods \cite{zhang2018collaborative, shu2018dirt, xu2020adversarial}.
Other metrics than the Wasserstein distance can be used to align domains including the MMD, in works such as \cite{zhang2021joint} that use pseudo labels for unlabeled data, with a manifold regularization; or the Bures-Wasserstein distance (a specific case of the Wasserstein distance on normal distributions) such as in \cite{liu2022bures}.
The method of \cite{xiao2021dynamic}, inspired by \cite{saito2018maximum}, considers the alignment between domains and the class discriminability simultaneously, and proposes to weight these two terms in the objective in a dynamic manner. Although departing from existing theoretical frameworks, it achieves state-of-the-art empirical performance.

Semi-Supervised Domain Adaptation assumes a setting where there are a well-labeled source domain and a partially-labeled target domain. Observing that a few target labels can greatly improve task performance in applications such as object detection and image recognition, Semi-Supervised Domain Adaptation has recently attracted attention. The methods in \cite{saito2019semi, qin2021contradictory} correct the classifier's predictions that are biased to the large amount of labeled data in the source domain by using conditional entropy computed from its predictions. In \cite{li2021ecacl, kim2020attract, jiang2020bidirectional}, the input data is perturbed by a powerful data augmentation (e.g.\ \cite{cubuk2020randaugment, devries2017improved}) or adversarial method (e.g.\ \cite{miyato2018virtual}), and then the model is trained so that the predictions for the original input and the perturbed input are consistent. Reference \cite{yang2021deep} proposes an efficient method for training a model by assigning pseudo one-hot labels to unlabeled target data predicted with high confidence during training. The methods presented above achieve good results in numerical experiments, but do not provide a rigorous theoretical discussion of the generalization error. In particular, \cite{saito2019semi} minimizes an upper-bound of the target error, but that upper-bound contains the joint minimum error that cannot be optimized, and therefore it is not guaranteed that the target error will necessarily be small after training. One such exception is JDIP \cite{chen2020domain}, which conducts a theoretical study with some similarities to ours, however in the case of domain adaptation, and with the classical L2 distance instead of a metric on distributions (the Wasserstein distance, with its advantages). This method builds on linear transformations and kernels models, whereas our approach works alongside more powerful and flexible neural networks.

Note that on both domain adaptation settings, the main goal is to improve generalization performance in the target domain, often at the expense of performance in the source domain. In our work, we would like to have high performance in \textit{all} given domains, and this task is better addressed using domain invariance.

\subsection{Domain Generalization}

Domain Generalization \cite{blanchard2011generalizing, muandet2013domain, zhou2022domain} is a very challenging problem that aims to achieve high performance on unseen target domains by learning models from multiple fully-labeled source domains. Domain generalization has received significant attention recently. Reference \cite{li2018domain} combines an adversarial loss with a maximum mean discrepancy regularizer in order to extract a representation where domains are aligned. The method of \cite{li2018deep} uses two adversarial losses to take advantage of label information in fully-labeled domains. The first loss matches the latent representation for each class, and the second loss reduces the negative effects of differences in class distributions across domains. The method of \cite{dou2019domain} uses meta-learning in order to extract features that are consistent across domains. Reference \cite{zhou2020domain} starts from an adversarial approach and incorporates a metric learning loss into the classifier in order to improve classification boundaries.
Reference \cite{meng2022attention} introduces a new \textit{attention diversification} framework, based on attention maps, where the latter are trained to produce diversified responses for task-related features and to remove domain specific features. The approach of \cite{zhou2021domain} mixes instance-level feature statistics across source domains. Mixing styles of training data has the effect of creating pseudo-new domains, resulting in increased diversity of training domains and improved generalization capability to unseen domains. The method of \cite{li2021divergence} can address unsupervised domain adaptation and model adaptation (or source-free unsupervised domain adaptation) as well as domain generalization. The method generates adversarial attacks to the extent that semantic information of original data is retained, and then learns to reduce the classification loss for those adversarial examples.

A few works have focused on theoretical aspects of Domain Generalization. Reference \cite{li2020domain} develops theoretical arguments based on a strong assumption that the distribution of latent variables in all domains is represented by a linear combination of other domains. Reference \cite{albuquerque2019generalizing} shows an upper-bound theorem indicating that minimizing the divergence between the source marginal distributions like \cite{ganin2016domain, shen2017wasserstein} can minimize the unseen target error when the target distribution exists in the neighborhood of the convex hull of source distributions. However, it is also known that minimizing the divergence to an extremely small value increases the divergence between the target distribution and the convex hull, which leads to an increase in the upper-bound. \cite{sicilia2021domain} derives a tighter upper-bound of the target error than \cite{albuquerque2019generalizing}. Note that the negative effect resulting from minimizing the divergence remains unresolved.

While domain generalization is a challenging and interesting topic on its own, it differs from the setting we consider in the present paper by requiring the model to generalize well to \textit{any} new domain. Not only this makes the theoretical analysis significantly harder than for the finite domain setting, performance gains on new unseen domains are often obtained at the expense of the existing domains, which are in our case the domains of interest.

\subsection{Domain Invariance}

In contrast to domain adaptation and domain generalization, domain invariance is a comparatively less explored setting. Domain invariance shares some technical similarities with works in distributionally robust optimization \cite{rahimian2019distributionally, duchi2018learning, duchi2020distributionally}. These works however focus on the optimization problem and its theoretical properties rather than the problem of generalization between different groups or domains. Another related area is subpopulation shifts, which addresses the question of generalization across predefined subgroups (e.g.\ \cite{sagawa2019distributionally, goel2020model, koh2020wilds}). Unlike domain invariance, works on subpopulation shifts focus on building invariance to subgroups of the same domain, often numerous with few samples and small discrepancies, rather than producing invariance to qualitatively different domains. Furthermore, works on subpopulation shifts typically consider that the different subgroups are fully labeled, whereas the domain invariance framework we introduce in our work enables learning with domains that are partially labeled. Due to the limited previous works and the lack of reference methods for domain invariance, our experiments section will resort to ablation studies for comparison.

\section{Domain Invariance and Optimal Transport}
\label{section:background}

Domain invariance can be described as the property of a representation to be indistinguishable with regards to its original domain, in particular, the multiple data distributions projected in representation space should look the same (i.e.\ have low distance). A recently popular way of measuring the distance between two distributions is the Wasserstein distance. The latter can be interpreted as the cost of transporting the probability mass of one distribution to the other if we follow the optimal transport plan, and it can be formally defined as follows:
\begin{definition}
Let $\mathcal{P}\in \mathcal{M}_{+}^1(\mathcal{A}), \mathcal{Q}\in \mathcal{M}_{+}^1(\mathcal{B})$ be two arbitrary probability distributions defined over two measurable metric spaces $\mathcal{A}$ and $\mathcal{B}$. Let $c$ be a cost function. Their Wasserstein distance is:
\begin{equation}
    W(\mathcal{P},\mathcal{Q}) \eqdef \inf\limits_{\pi \in \Pi(\mathcal{P},\mathcal{Q})} \int_{\mathcal{A}\times\mathcal{B}}c(a,b)d\pi(a,b)
\label{eq:primal}
\end{equation}
with $\Pi(\mathcal{P},\mathcal{Q})\eqdef\{\pi \in \mathcal{M}_{+}^1(\mathcal{A}\times\mathcal{B}):P_{\mathcal{A}\#}\pi = \mathcal{P} \text{ and } P_{\mathcal{B}\#}\pi = \mathcal{Q}\}$, where $\bf{P}_{\mathcal{A}\#}$ and $\bf{P}_{\mathcal{B}\#}$ are push-forwards of the projection of $P_{\mathcal{A}}(a,b) = a$ and $P_{\mathcal{B}}(a,b) = b$. This can be loosely interpreted as $\Pi(\mathcal{P},\mathcal{Q})$ being the set of joint distributions that have marginals $\mathcal{P}$ and $\mathcal{Q}$.
\end{definition}
Hence, we measure the invariance of representations by how low the Wasserstein distance is between the distributions $\mathcal{P}$ and $\mathcal{Q}$ associated to the two domains. The $\mathcal{P}$ and $\mathcal{Q}$ distributions respectively are the $\mathcal{P}_1$ and $\mathcal{P}_2$ distribution of Fig.\ \ref{fig:comparison-tasks}. The Wasserstein distance being scale-dependent, we assume that representations of both domains have fixed scale. In comparison to other common alternatives such as the Kullback-Leibler divergence, the Jensen-Shannon divergence, or the Total Variation distance, the Wasserstein distance has the advantage of taking into account the metric of the representation space (via the cost function $c(a,b)$), instead of looking at pure distributional overlap, and this typically leads to better ML models \cite{DBLP:conf/nips/MontavonMC16,arjovsky2017wasserstein}. Computing the Wasserstein distance with Eq.\ \eqref{eq:primal} is expensive. Luckily, if we use the metric of our space as a cost function, such as the Euclidean distance $c(a,b)=\norm{a-b}_2$, we can derive a dual formulation of the 1-Wasserstein distance as follows:
 \begin{equation}
     W(\mathcal{P},\mathcal{Q}) = \sup\limits_{\norm{\varphi}_{Lip} \leq 1} \mathbb{E}_{\mathcal{P}}[\varphi] - \mathbb{E}_{\mathcal{Q}}[\varphi]
     \label{eq:wlipschitz}
 \end{equation}
where $\mathbb{E}_\mathcal{P}[\varphi]$ is the expected value of function $\varphi$ on the distribution $\mathcal{P}$. This formulation replaces an explicit computation of a transport plan, by a function to estimate, a task particularly appropriate for neural networks. Recently several methods have used this approach to learn distributions \cite{DBLP:conf/nips/MontavonMC16} specifically in the context of Generative Adversarial Networks \cite{arjovsky2017wasserstein,shen2017wasserstein}. The main constraint lies in the necessity of the function $\varphi$, which we will call the discriminator (or critic), to be 1-Lipschitz. A few approaches were proposed to tackle this problem, such as gradient clipping \cite{arjovsky2017wasserstein}, gradient penalty \cite{gulrajani2017improved} and more recently the Spectral Normalization \citep{miyato2018spectral}. It is however important to note that in practice the set of possible discriminators will be a subset of 1-Lipschitz continuous functions.

\section{Relating Wasserstein Distance to Supervised Losses}
\label{section:theory}

We would like to align the predicting behavior of a ML model on multiple domains following the approach illustrated in Figure \ref{fig:intro}, i.e.\ by learning a \textit{domain-invariant representation}. Specifically, we aim for a representation of data where the distributions associated to the two domains have minimum Wasserstein distance and therefore cannot be distinguished. At the same time, the representation should contain the features that are necessary to solve the given prediction task, e.g.\ using common supervised loss functions. We focus here on the two-domain case, and refer to Supplementary Note D for the case of three or more domains.

Let us start with some formalities: We denote by $\mathcal{X}$ the input space, by $\mathcal{Z} \subset \R^d$ our representation or feature space, and by $\mathcal{Y}$ the label or target space. We further denote by $\Phi:\mathcal{X} \to \mathcal{Z}$ the feature extractor, and by $f:\mathcal{Z} \to \mathcal{Y}$ the prediction function (e.g.\ regression; classification). We assume $\mathcal{Z}$ and $\mathcal{Y}$ to be compact measurable spaces, and we denote by $\mathcal{M}^1_{+}(\mathcal{Z}\times\mathcal{Y)}$ the set of probability distributions defined on their product space. Let  $\mathcal{P}^t, \mathcal{Q}^t \in \mathcal{M}^1_{+}(\mathcal{Z}\times\mathcal{Y)}$ be the \emph{true} probability distributions formed by the two domains we would like to align. When necessary, we add a subscript to these distributions to specify their support.

Similarly to previous works, domain alignment will be measured as the Wasserstein distance $W(\mathcal{P}^t, \mathcal{Q}^t)$ of samples embedded in \emph{feature space}, but also including their labels. We contribute by showing that the Wasserstein distance $W(\mathcal{P}^t, \mathcal{Q}^t)$, can be related formally to common loss functions used in classification or regression, via mathematical inequalities. With these inequalities one can design practical learning objectives fairly easily, whose minimization not only solves the task at hand, but also implies as a side effect the minimization of the Wasserstein distance between distributions of the two domains, thereby achieving domain invariance.

From a technical standpoint, our novel approach draws some inspiration from \cite{courty2017joint} and is based on measure theory, in order to formalize partially labeled distributions and therefore our problem of aligning multiple joint distributions. The individual steps are presented in Sections \ref{section:ssup}--\ref{section:crossentropy}, and we provide an overview of our novel theoretical framework in \mbox{Figure \ref{fig:theory}}.

\begin{figure}[b!]
    \centering
    \includegraphics[width=.9\columnwidth]{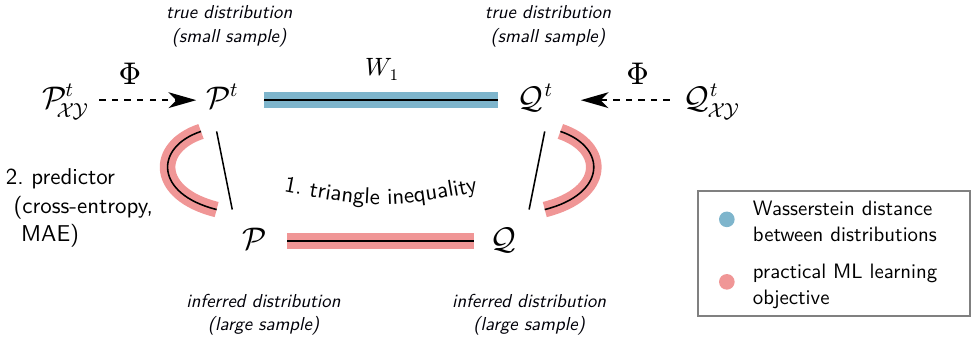}
    \caption{Visual overview of our theoretical framework. It relates the Wasserstein distance between the joint distributions $\mathcal{P}^t,\mathcal{Q}^t$ of each domain (blue) to components of a practical ML objective (red) in two steps. Step 1: The true joint distributions (of which only a small sample is observable) can be related via the triangle inequality to inferred distributions $\mathcal{P},\mathcal{Q}$ where missing labels are \textit{predicted} from the features. Step 2: The expanded terms can be further upper-bounded by common terms of a ML objective (cross-entropy, mean absolute error (MAE), etc.).}
    \label{fig:theory}
\end{figure}

\subsection{Incorporating Semi-Supervised Data}
\label{section:ssup}

Computation of the true Wasserstein distance $W(\mathcal{P}^t,\mathcal{Q}^t)$ would require knowledge of the true distributions $\mathcal{P}^t$ and $\mathcal{Q}^t$. In practice, we only have a finite sample of these distributions, and the quality with which the Wasserstein distance can be approximated largely depends on the amount of labeled data available. (For high-dimensional tasks, the necessary amount of labels would be overwhelming.) However, in practice, it is common that \textit{unlabeled data} is available in much larger quantity than the labeled data. We consider this semi-labeled scenario where only a fraction of data are obtained from  $\mathcal{P}^t,\mathcal{Q}^t$ (i.e.\ labeled). The remaining data are unlabeled and obtained from the marginals $\mathcal{P}^t_\mathcal{Z},\mathcal{Q}^t_\mathcal{Z}$.

By learning an appropriate function $f: \mathcal{Z} \rightarrow \mathcal{Y}$, say, a neural network classifier, that infers labels from features, one can obtain an approximation $\mathcal{P}^f=(z, f(z))_{z\sim \mathcal{P}^t_\mathcal{Z}}$ of the true joint distribution $\mathcal{P}^t$. This implies that the distribution we effectively have access to (and draw from) in practice is the mixture:
$$
\mathcal{P}=\alpha \mathcal{P}^t + (1 - \alpha) \mathcal{P}^f.
$$
where $\alpha \in(0,1)$ is the fraction of labeled data. We will refer to $\mathcal{P}$ as the \textit{inferred} distribution. Identically for the second domain, we construct an appropriate function $g: \mathcal{Z} \rightarrow \mathcal{Y}$ from which one can predict labels, and which results in another inferred distribution:
$$
\mathcal{Q}=\beta \mathcal{Q}^t + (1 - \beta) \mathcal{Q}^g.
$$
Note that the functions $f$ and $g$ need not be identical. Also, $\beta$ may differ from $\alpha$, which addresses the case where different domains have different proportions of labeled data.
Let the Wasserstein distance's cost function $c$ be the metric on the space $\mathcal{Z}\times\mathcal{Y}$. By application of the triangle inequality, the following relation can be extracted: 
\begin{equation}
    W(\mathcal{P}^t, \mathcal{Q}^t) \leq W(\mathcal{P}^t, \mathcal{P}) + W(\mathcal{P}, \mathcal{Q}) + W(\mathcal{Q}, \mathcal{Q}^t).
    \label{eq:triangle}
\end{equation}
In other words, the distance between inferred distributions $W(\mathcal{P}, \mathcal{Q})$ to which we add the inference `errors', i.e.\ the distance between true and inferred distributions on each domain, form an upper-bound to the true distance between distributions. Let us now analyze the error term $W(\mathcal{P}^t, \mathcal{P})$. We consider the case of $\mathcal{P}$ (analogously so for $\mathcal{Q}$).
\begin{lemma}
\label{lemma:wasserstein_mixture_decomposition}
Let $\mathcal{K}\in\mathcal{M}_{+}^1(\mathcal{Z}\times\mathcal{Y})$ be an arbitrary probability distribution, we then have
$ W(\mathcal{P},\mathcal{K})\leq \alpha W(\mathcal{P}^t, \mathcal{K}) + (1-\alpha)W(\mathcal{P}^f, \mathcal{K}),
$ and for the special case where $\mathcal{K}=\mathcal{P}^t$, we get
\begin{equation}
    W(\mathcal{P},\mathcal{P}^t) = (1-\alpha)W(\mathcal{P}^f, \mathcal{P}^t).
    \label{eq:lemma1-specific}
\end{equation}
\end{lemma}
A proof can be found in Supplementary Note A. The proof proceeds by first decomposing $\mathcal{P}$ into its mixture components, and then applying Jensen's inequality on the Wasserstein dual's supremum. For the special case of Eq.\ (4), the equality is due to $\mathcal{K}$ being an element of the mixture $\mathcal{P}$ and the Wasserstein distance of $\mathcal{K}$ to that mixture element being consequently zero.

Finally, combining the results above, that is, by applying the triangle inequality, Lemma \ref{lemma:wasserstein_mixture_decomposition}, and using the symmetry property of the Wasserstein distance, one can obtain another bound on the true Wasserstein distance, where unlike Eq.\ \eqref{eq:triangle}, some mixture components now appear explicitly:
\begin{theorem}
\label{th_final}
Given the Wasserstein distance's cost function $c$ is the metric on the product space $\mathcal{Z}\times\mathcal{Y}$, we get
\begin{equation}
    W(\mathcal{P}^t, \mathcal{Q}^t) \leq  (1-\alpha)W(\mathcal{P}^t,\mathcal{P}^f)
    + W(\mathcal{P},\mathcal{Q}) +  (1-\beta)W(\mathcal{Q}^t,\mathcal{Q}^f).
    \label{eq:theorem}
\end{equation}
\end{theorem}
This final formulation will let us relate in Section \ref{section:crossentropy} some of the expanded terms, specifically, the distances $W(\mathcal{P}^t,\mathcal{P}^f)$ and $W(\mathcal{Q}^t,\mathcal{Q}^f)$ to common loss functions used in supervised machine learning.

\subsection{Connection to Supervised ML Losses}
\label{section:crossentropy}

Various loss functions have been proposed for supervised learning. They address the diversity of output types (e.g.\ class labels; regression targets) and statistical properties of the data (e.g.\ margin between classes; presence/absence of outliers). Ideally, one would be able to achieve domain invariance while retaining the ability to optimize the most suitable loss function for a given problem.

\smallskip

Let us start with Eq.\ (5) in Theorem \ref{th_final}, in particular, the distance $W(\mathcal{P}^t,\mathcal{P}^f)$. The latter essentially measures the level of error with which the function $f$ predicts the true labels $y$. It therefore plays a similar role to common loss functions used for supervised machine learning. Both can also be mathematically related. First, the Mean Absolute Error (MAE) commonly used in robust regression can be related to $W(\mathcal{P}^t,\mathcal{P}^f)$ as follows:
\begin{lemma}
\label{lemma:mse_bound} 
\begin{equation}
    \begin{aligned}
W(\mathcal{P}^{t}, \mathcal{P}^{f}) &\leq\mathbb{E}_{(z,y) \sim \mathcal{P}^t}\big[|y-f(z)|\big]
\end{aligned}
\end{equation}
\end{lemma}
The proof is given in Supplementary Note A. In the case of classification, a similar result can be provided for the Kullback-Leibler (KL) divergence, which is equivalent to the cross-entropy loss when $\mathcal{P}_{\mathcal{Y}|z}$ is deterministic:
\begin{lemma}
\label{lemma:true_estimated_dist} 
Assuming that $\mathcal{P}^f$ and $\mathcal{P}^t$ admit densities, we then obtain
\begin{equation}
    W(\mathcal{P}^t, \mathcal{P}^f)\leq \operatorname{diam}(\mathcal{Z}\times\mathcal{Y})\sqrt{\frac{1}{2}\mathbb{E}_{z\sim\mathcal{P}_\mathcal{Z}}\left[\KL{\mathcal{P}^t_{\mathcal{Y}|z}}{\mathcal{P}^f_{\mathcal{Y}|z}}\right]},\\
\end{equation}
where $\KL{\mathcal{P}}{\mathcal{Q}}=-\int_{\mathcal{Z}\times\mathcal{Y}} \log \frac{d \mathcal{Q}}{d \mathcal{P}} d \mathcal{P}$ is the Kullback-Leibler divergence; and where $\operatorname{diam}(\cdot)$ is the diameter of the space received as input, i.e.\ the largest distance obtainable in that space.
\end{lemma}

For a proof, see Supplementary Note A. Lemmas \ref{lemma:mse_bound} and \ref{lemma:true_estimated_dist} now relate the Wasserstein distance formulation to loss functions occurring in regression and classification tasks that are easily computable, and with the desired statistical properties. Together with the relation shown in Section \ref{section:ssup}, we can now propose a ML formulation that both addresses the prediction task, and enforces domain invariance.

\section{Learning a Domain Invariant Neural Network}
\label{section:dinn}

Consider the data available consists of examples sampled from both domains, specifically, from distributions $\mathcal{P}$ and $\mathcal{Q}$. Under these distributions, part of the data comes with the true labels. For the rest of the data, labels are inferred via the functions $f$ and $g$ respectively. We denote by $(X^p, Y^p)$ the dataset of $n$ examples drawn from the first domain $\mathcal{P}$ and by $(X^q, Y^q)$ the dataset of $m$ examples drawn from the second domain $\mathcal{Q}$.
Based on Theorem \ref{th_final} and Lemmas \ref{lemma:mse_bound}--\ref{lemma:true_estimated_dist}, one can define a learning procedure that consists of simultaneously minimizing a supervised loss function $\mathcal{L}$ on each domain and the Wasserstein distance $W(\mathcal{P},\mathcal{Q})$ aligning distributions of the two domains. In the classification setting, the supervised loss for the first domain is defined as
\begin{equation}
    \mathcal{L}(Z^p,Y^p) = \frac{1}{n}\sum_{i = 1}^{n} \KL{f(Z_i^p)}{Y_i^p},
     \label{objective:supervised}
\end{equation}
where $Y_i^p$ and $f(Z_i^p)$ are vectors containing probabilities of each class. A similar loss function can be built for the the second domain. Note that the loss is only effective on the examples that come with a true label, because when the label is inferred, we have $\KL{f(Z_i^p)}{Y_i^p} = \KL{f(Z_i^p)}{f(Z_i^p)} = 0$. In a similar fashion, for the regression setting, we define $\mathcal{L}(Z^p, Y^p) = \frac{1}{n}\sum_{i = 1}^{n} \abs{f(Z^p_i) - Y^p_i}$. For the minimization of the Wasserstein distance $W(\mathcal{P},\mathcal{Q})$ we consider the dual form provided in Eq.\ \eqref{eq:wlipschitz}, specifically, an empirical estimate of it: 
\begin{equation}
    W(\mathcal{P},\mathcal{Q}) \approx \max_{\varphi \colon  \norm{\varphi}_{Lip} \leq 1} \Big\{\underbrace{\frac{1}{n}\sum_{i = 1}^n \varphi(Z^p_i, Y^p_i)
     - \frac{1}{m}\sum_{i = 1}^m  \varphi(Z^q_i, Y^q_i)
     }_{\displaystyle \big. \Delta(Z^p, Y^p, Z^q, Y^q) \big.}\Big\}
     \label{objective:wasserstein}
\end{equation}
and this forms our domain critic. Finally, one can sum the supervised terms and the domain critic, i.e.\ Eqs.\ \eqref{objective:supervised} and \eqref{objective:wasserstein}, and optimize the resulting objective w.r.t.\ the functions $f$, $\Phi$, $\varphi$ (more precisely, the parameters of the neural networks implementing these functions). This can be formulated as the GAN-like optimization problem:
\begin{equation}
    \begin{split}
        \min\limits_{\Phi,f}\max\limits_{\varphi}\quad & \Big\{\Delta(\Phi(X^p), Y^p, \Phi(X^q), Y^q) \\[-1mm]&\qquad + \lambda_p\mathcal{L}(f(\Phi(X^p)), Y^p) \\[.5mm]&\qquad\qquad + \lambda_q\mathcal{L}(f(\Phi(X^q)), Y^q)\Big\} \qquad  \text{s.t. } \norm{\varphi}_{Lip} \leq 1.
    \end{split}
    \label{eq:objective}
\end{equation}
where the classifier terms and the domain critic are in competition. The hyperparameters $\lambda_p$ and $\lambda_q$ can either be fixed to $1-\alpha$ and $1-\beta$ respectively (and for classification multiplied by the domain's diameter) in order to match the theory; or they can be selected heuristically or based on some validation procedure. The Lipschitzness constraint on $\varphi$ is practically enforced by using one of the regularization techniques mentioned at the end of Section \ref{section:background}. Additionally, a constraint on the scale of the representation or the Lipschitzness of the classifier $f$ can be added in order to prevent an arbitrary downscaling of the representation which may cause the Wasserstein distance to artificially go to zero. Lastly, supplementary regularization terms, such as EntMin \cite{grandvalet2005semi}, Virtual Adversarial Training \cite{miyato2018virtual}, and Virtual Mixup \cite{mao2019virtual} can be added to the objective, in order to take further advantage of the unlabeled examples. A visual representation of our model is given in Figure \ref{fig:diagram}.
The steps needed to train a domain invariant networks are summarized in Algorithm \ref{alg:two}.

\begin{figure}
    \centering
    \includegraphics[width=\columnwidth]{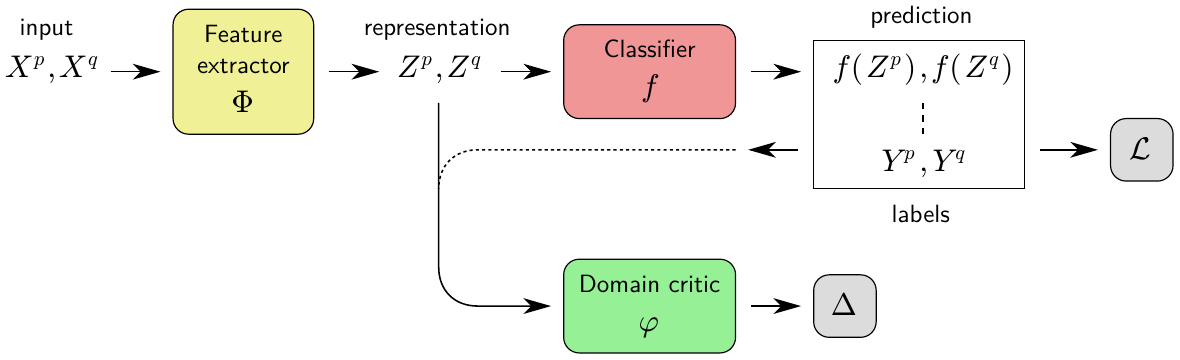}
    \caption{Diagram of the proposed machine learning model, that induces a domain-invariant representation through a domain critic.}
    \label{fig:diagram}
\end{figure}

\begin{algorithm}
\caption{Algorithm for training our proposed domain invariant network. The function `$\mathrm{Rev\_Grad}$' denotes a gradient reversal layer, which leaves the forward pass unchanged but multiplies the gradient by $-1$ in the backward pass (see e.g.\ \cite{ganin2016domain}).}\label{alg:two}
\KwData{Semi-supervised datasets for both domains: $(X^p,Y^p),(X^q,Y^q)$}
\KwIn{Untrained $\Phi, f$ with parameters $\theta$, hyperparameters $\lambda_{p}, \lambda_{q}$}
\KwResult{Trained $\Phi$, $f$}
\For{$\mathrm{epochs}$}{
    \For{$\mathrm{batch}$ $(x^p,y^p),(x^q,y^q)\in (X^p,Y^p),(X^q,Y^q)$}{
        \Comment{Compute features}
        $z^p \gets \Phi(x^p)$\\
        $z^q \gets \Phi(x^q)$\\
        \Comment{Impute instances with missing labels in the batch}
        $y^p_{i} \gets f(z^p_{i})~~\forall~\mathrm{unlabeled}~i$\\
        $y^q_{j} \gets f(z^q_{j})~~\forall~\mathrm{unlabeled}~j$\\
        \Comment{Reverse gradient of features for domain discriminator}
        $z^p_{\mathrm{rev}},y^p_{\mathrm{rev}} \gets \operatorname{Rev\_Grad}(z^p),\operatorname{Rev\_Grad}(y^p)$\\
        $z^q_{\mathrm{rev}},y^q_{\mathrm{rev}} \gets \operatorname{Rev\_Grad}(z^q),\operatorname{Rev\_Grad}(y^q)$\\
        \Comment{Compute losses}
        $L_\mathrm{disc} \gets \Delta(z^p_{\mathrm{rev}}, y^p_{\mathrm{rev}}, z^q_{\mathrm{rev}}, y^q_{\mathrm{rev}})$\\
        $L_\mathrm{classif} \gets \lambda_p\mathcal{L}(f(z^p)), y^p) + \lambda_q\mathcal{L}(f(z^q)), y^q)$\\
        $L_\mathrm{total} \gets L_\mathrm{disc} + L_\mathrm{classif}$\\
        \Comment{Perform a gradient descent step}
        $\theta \gets \theta - \gamma \nabla L_\mathrm{total}$
}}
\end{algorithm}

We now outline some specific condition on the data distribution which provides statistical consistency of the estimator.
\begin{remark}
Let us consider the condition on which Eq.\ \eqref{eq:objective} is statistically consistent when sufficiently many labeled examples are observed. 
For the data distributions $\mathcal{P}_{\mathcal{XY}}^t=\mathcal{P}_{\mathcal{Y|X}}^t \mathcal{P}_{\mathcal{X}}^t$ and $\mathcal{Q}_{\mathcal{XY}}^t=\mathcal{Q}_{\mathcal{Y|X}}^t        \mathcal{Q}_{\mathcal{X}}^t$, suppose that there exists a function $z=\Phi^o(x)$ such that
\begin{enumerate}[(i)]
    \item the conditional probabilities satisfy 
    $\mathcal{P}_{\mathcal{Y|X}}^t(y|(\Phi^o)^{-1}(A))=\mathcal{Q}_{\mathcal{Y|X}}^t(y|(\Phi^o)^{-1}(A))$ for any measurable subset $A\subset \mathcal{Z}$, and 
    \item $\Phi_*^o P_{\mathcal{X}} = \Phi_*^o Q_{\mathcal{X}}$,  
\end{enumerate}
where $\Phi_*^o$ is the push-forward with the function $x\mapsto \Phi^o(x)$. 
Under the above assumptions, we see that 
$\Phi_*^o\mathcal{P}_{\mathcal{XY}}^t=\Phi_*^o\mathcal{Q}_{\mathcal{XY}}^t$ holds. Hence, the Wasserstein distance estimator $\Delta$ under the population distribution becomes zero. Due to the assumption on the conditional distribution, the optimal classifier $f$ on $\mathcal{Z}$ is common  for both distributions.  
Therefore, the optimal classifier with domain-invariant features is obtained by Eq.\ \eqref{eq:objective}.
\end{remark}
To put it more simply, by assuming there exists a feature map where marginal distributions are aligned, and that the conditional distributions are equal almost everywhere for the image of $\Phi^o(x)$, optimizing our objective function leads to the optimal classifier.

\subsection{Generalization Bounds}

Interestingly, the Wasserstein distance between the true distributions of the two domains (that we have upper-bounded in Theorem \ref{th_final}) can also be related to the risks of the classifier on the two domains. Let $\mathcal{R}_{\mathcal{P}^t}(f)=\mathbb{E}_{z,y\sim\mathcal{P}^t}\left[\mathcal{L}(f(z),y)\right]$ be the risk or error of a classifier $f$. We here develop a result using the joint Wasserstein distance, similar to previous result obtained by \citep{redko2017theoretical} on the distance between marginals.
\begin{theorem}
\label{th_diff_err}
Let $\mathcal{Z},\mathcal{Y}$ be two compact measurable metric spaces whose product space has dimension $d$. Let $\mathcal{P}^t,\mathcal{Q}^t\in\mathcal{M}_{+}^{1}(\mathcal{Z}\times\mathcal{Y})$ two joint distributions associated to the two domains, and $\widehat{\mathcal{P}}^t,\widehat{\mathcal{Q}}^t$ their empirical counterparts. Let the transport cost function $c$ associated to the optimal transport problem be $c(z_1,y_1;z_2,y_2)=\norm{\veccol{z_1 \\ y_1}-\veccol{z_2 \\ y_2}}_2$, the Euclidean distance as the metric on $\mathcal{Z}\times\mathcal{Y}$ and $\mathcal{L}:\mathcal{Y}\times\mathcal{Y}\rightarrow\R^{+}$ a symmetric $\kappa$-Lipschitz loss function. Then for any $d^\prime > d$ and $\psi^\prime<\sqrt{2}$ there exists some constant $N_0$ depending on $d^\prime$ such that for any $\delta > 0$ and $\min(N_P, N_Q)\geq N_0\max(\delta^{-(d^\prime+2)},1)$ with probability at least $1-\delta$ for all $\lambda$-Lipschitz $f$ the following holds:
\begin{equation}
    \abs{\mathcal{R}_{\mathcal{Q}^t}(f)-\mathcal{R}_{\mathcal{P}^t}(f)} \leq \kappa\sqrt{(\lambda^2 +1)}\left[W_{1}(\widehat{\mathcal{P}}^t,\widehat{\mathcal{Q}}^t)+\sqrt{\frac{2}{\psi^{\prime}}\log \left(\frac{1}{\delta}\right)}\left(\sqrt{\frac{1}{N_{P}}}+\sqrt{\frac{1}{N_{Q}}}\right)\right].
\end{equation}
\label{theorem:gap}
\end{theorem}

(A proof is given in Supplementary Note A.) In other words, the empirical Wasserstein distance between the two domains upper-bounds the prediction performance gap between the two domains. In practice, we can therefore expect the optimization of the objective in Eq.\ \eqref{eq:objective} to not only reduce the Wasserstein distance between domains (as we have shown in the previous sections), but also to produce a more uniform classification accuracy across domains and therefore a higher minimum accuracy.

We may also want to compare the joint discriminator $W(\mathcal{P},\mathcal{Q})$ to the more common marginal discriminator $W(\mathcal{P}_\mathcal{Z},\mathcal{Q}_\mathcal{Z})$. Indeed, it seems that in many cases, such as when the conditional distribution is identical between the two domains, the solution obtained appears to be equivalent. This is true due to theoretical reasons we will explore here. Let us first recall a bound on the distance between errors with a marginal Wasserstein distance.
\begin{theorem}[From \cite{shen2017wasserstein} (Adapted)]
 Let $\mathcal{P}^t_\mathcal{Z}, \mathcal{Q}^t_\mathcal{Z} \in \mathcal{M}_{+}^{1}(\mathcal{Z})$ be two probability measures. Assume the functions $f \in H$ are all $\lambda$-Lipschitz continuous for some $\lambda$. Then for every $f \in H$ the following holds
$$
\abs{\mathcal{R}_{\mathcal{P}^t}(f)-\mathcal{R}_{\mathcal{Q}^t}(f)} \leq 2 \lambda \cdot W_{1}\left(\mathcal{P}^t_\mathcal{Z}, \mathcal{Q}^t_\mathcal{Z}\right)+\mathcal{E}
$$
where $\mathcal{E}$ is the combined error of the ideal $f^{*}$ that minimizes the combined error $\mathcal{R}_{\mathcal{P}^t}(f)+\mathcal{R}_{\mathcal{Q}^t}(f)$.
\end{theorem}
The main difference compared with our bound is the presence of $\mathcal{E}$, the combined error of the hypothesis on both domains. Indeed, when the features of the two domains are properly aligned, the bound obtained with a joint or marginal Wasserstein distances are similar. However, when the domains are not properly aligned, usually due to the transformation between the two domains being large, and to the lack of labeled samples, we can have a large $\mathcal{E}$ such as $\mathcal{E}=1$, which renders the bound very large. Such a case arises when both domains have entirely identical samples but with opposite labels, for instance. The bound with the joint Wasserstein distance can lead to features more aligned even with large transformations between domains. We expect to observe similar performances between marginal and joint discriminators on experiments with simple transformations, and larger discrepancies as the transformations get larger.

\section{Experiments}

To test whether our proposed approach truly achieves an invariant representation and reduces the performance gap between domains (as predicted by Theorems \ref{th_final} and \ref{th_diff_err} respectively), we conduct experiments on three common image classification problems. First, a handwritten digits recognition task where the digit images come from two popular datasets: MNIST \cite{lecun1998mnist} and SVHN \cite{netzer2011svhn}, each of them constituting one domain. Then, we consider the Office-Caltech classification dataset \citep{gong2012geodesic}, which consists of four domains. Finally, we consider the recent and more complex PACS multi-domain image recognition dataset \cite{li2017deeper}, which also consists of four domains. We describe below these multi-domain tasks, and the training procedure for our models. More details are provided in Supplementary Note B.

\subsection{Data and Models}

MNIST and SVHN are two common digit recognition datasets composed of 60000 and 73257 training examples respectively. While MNIST digits are black\&white, SVHN digits are colored and have more complex appearances, making them harder to predict. In our MNIST-SVHN two-domain scenario, we simulate partly labeled data by only providing labels for a random subset of examples (1000 for each domain for the experiments of Table \ref{table:main}, and 3000 per domain for experiments of Table \ref{table:semisup}). The remaining examples are given unlabeled. MNIST images are brought to the SVHN format by scaling and setting each RGB component to the MNIST grayscale value. For experiments in Tables \ref{table:main} and \ref{table:semisup}, the function $\Phi$ is implemented by the Conv-Large model from \cite{miyato2018virtual}. The model takes as input images of size $32 \times 32 \times 3$. We use small random translations of 2 pixels as well as color jittering as data augmentation.

Importantly, for the purpose of evaluating the domain invariance of representations, we would like to stabilize the scale of representations learned by the different models. Specifically, we add for the experiments of Table \ref{table:main} a further penalty to the objective: the Wasserstein distance between the distribution of distances in representation space (the histogram of distances of the union distribution of $\mathcal{P}_\mathcal{Z}$ and $\mathcal{Q}_\mathcal{Z}$) and a predefined Gaussian mixture, which we set to be a univariate mixture of Gaussians $\frac{1}{10}\mathcal{N}(5, 2)+\frac{9}{10}\mathcal{N}(15,3)$.
The two modes model distances between data points of same class and of different classes respectively. Since the distribution of distances is a 1-dimension histogram, the Wasserstein distance can be computed analytically \cite{peyre2019computational}. This added constraint ensures a similar scale for the representation extracted by our model and the different baselines. In particular, it ensures that a reduction of Wasserstein distance in representation space can be reliably interpreted as an increase of domain invariance, and not as a simple scaling of the representation. We have experimented with several new metrics and constraints for settling for this one, although it comes with some side effects. Indeed, by its very definition, it implies that there should be 10 equidistant and equally sized clusters, which is an assumption that is not verified for all datasets (for instance, SVHN). Moreover, it gives a stronger advantage (in the form of a prior) to the bare methods, without any discriminator, acting indirectly as one.

Our second scenario is based on the Office-Caltech dataset. It is composed of four domains (Amazon, Caltech, DSLR, Webcam) containing pictures of objects present in offices (such as monitors) from different sources, such as pictures from a real office, or ones with white backgrounds from an e-commerce website. There are 10 classes, and a varying number of samples depending on the domain (between 150 and 1100). We use the Decaf6\citep{donahue2014decaf} features with 4096 dimensions. We use the Resnet-18 architecture \cite{he2016deep}. We train a model on each possible bi-domain task (6 tasks) and average the resulting accuracies per domain.

Our third and last  scenario is based on the PACS dataset, which consists of 10000 examples, with 4 domains (Photo, Art, Cartoon, Sketch) and 7 classes. We simulate semi-labeled data by providing labels for only 500 randomly sampled images from each domain, and giving remaining images unlabeled. The classes and domains are imbalanced, i.e.\ contain a different number of examples. The images are resized to $224\times224\times3$, and a pipeline of data augmentation is applied based on RandAugment \cite{cubuk2020randaugment}. We again use the Resnet-18 architecture. On this dataset, we test domain invariance in a `one vs.\ rest' setting.

\smallskip

In all our experiments, the classifier $f$ is a simple 2-layer MLP, and the discriminator $\varphi$ a 3-layer MLP with spectral normalized weights \cite{miyato2018spectral}. (On the multi-domain PACS, we use a discriminator for each domain, computed in a one-vs-rest manner.) The weights (hyperparameters) for each loss term $\lambda_p$ and $\lambda_q$ are set to one, as well as the discriminator's. Unless mentioned otherwise, the networks are trained for 20 to 50 epochs using the Adam \cite{kingma2014adam} optimizer.

\subsection{Results and Analysis}
\label{section:results}
As a first experiment, we study the effect of the domain critic we have proposed in Section \ref{section:dinn} on the accuracy of the model, and on the Wasserstein distance between the two domains. We consider two baselines for comparison: (1) a simple supervised neural network without domain critic, (2) a supervised network where the critic $\varphi$ is based only on \textit{marginal} distributions (such as proposed in \citep{shen2017wasserstein}). These two baselines can be interpreted as an ablation study of our method, where instead of applying the Wasserstein distance to the joint input-label distribution, we apply it first only to the input variables (marginal critic), and then to no variables at all (no critic). For this experiment we do not use any additional losses/regularizers, and simply optimize the classification and domain alignment terms. We report the Wasserstein distance between the two domains' joint distributions, and the minimum classification accuracy for the two domains. These are two properties that our domain-invariant network is expected to fulfill (Theorems \ref{th_final} and \ref{theorem:gap} respectively). Results are shown in Table \ref{table:main}.

\begin{table}[h]
    \centering
    \begin{tabular}{p{.4\textwidth}|cc|cc|c}
    \toprule
    & \multicolumn{4}{c|}{Accuracy} & \\
        Model  &  \textit{MNIST} & \textit{SVHN} & Avg & Min & W dist. \\
        \midrule
        No critic & \textit{98.9} & \textit{90.2} & 94.55 & 90.2 & 3.92\\
        Marginal critic \citep{shen2017wasserstein} & \textit{97.5} & \textit{91.5}  & 94.5 & \textbf{91.5} & 3.43\\
        Joint critic (Ours) & \textit{97.5} & \textit{91.5}  & \textbf{94.6} & \textbf{91.5} & \textbf{3.36} \\
        \bottomrule
    \end{tabular}
    \caption{Effect of the domain critic on the classification accuracy and the Wasserstein distance between the two domains in representation space. We use 1000 labels per domain. Best performance is shown in bold. For indicative purpose, we report in the first two rows the classification accuracy on individual domains.}
    \label{table:main}
\end{table}

Results corroborate our theory. In particular, we observe that the Wasserstein distance significantly decreases under the effect of adding a domain critic, specifically a joint domain critic that puts more focus on $\mathcal{Y}$, and the minimal accuracy over the two domains increases. Furthermore, we observe in this experiment that the use of a joint critic also leads to the highest average accuracy across domains.

\smallskip

Independently of the question of domain invariance, unsupervised data has already been routinely leveraged by classical semi-supervised learning approaches. These approaches have shown powerful on data with manifold structure (e.g.\ \cite{DBLP:conf/nips/RasmusBHVR15,li2017triple}). In our next experiment, we test the benefit of domain alignment techniques on models that are already equipped with semi-supervised learning mechanisms. Specifically, we consider a combination of two common semi-supervised techniques: conditional entropy minimization (EntMin) \cite{grandvalet2005semi} and virtual adversarial training (VAT) \citep{miyato2018virtual}, which have shown strong empirical performance on numerous tasks. Results are given in Table \ref{table:semisup}.

\begin{table}[h]
    \centering
    \begin{tabular}{l|cc|cc}
    \toprule
    & \multicolumn{4}{c}{Accuracy} \\
        Model  &  \textit{MNIST} & \textit{SVHN} & Avg & Min \\
        \midrule
        \textit{No critic + VAT/EntMin (only MNIST)} & \textit{99.14} & $\cdot$ & $\cdot$ & $\cdot$\\
        \textit{No critic + VAT/EntMin (only SVHN)} & $\cdot$ & \textit{94.79}  & $\cdot$ & $\cdot$ \\
        \midrule
        No critic & \textit{98.76} & \textit{87.33} & 93.05 & 87.33\\
        No critic + VAT/EntMin & \textit{99.29} & \textit{91.86} & 95.58 & 91.86 \\
        \midrule
        Joint critic (Ours) + VAT/EntMin &\textit{99.26} & \textit{92.75}  & 96.01 & 92.75\\
        Joint critic (Ours) + VAT/EntMin + Fine-tuning & \textit{99.09} & \textit{94.33}  & \bf{96.71} & \bf{94.33}\\
        \bottomrule
    \end{tabular}
    \caption{Evaluation of our method in combination with classical semi-supervised learning regularizers (VAT\,+\,EntMin). We use 3000 labels per domain. Best results are in bold.}
    \label{table:semisup}
\end{table}

We observe that semi-supervised learning on both domains, achieved by a combination of VAT and EntMin, leads to a strong baseline. In particular, it achieves the highest performance on MNIST. Our domain invariant approach, combined with the same techniques, further improves over this strong baseline, by reducing the accuracy gap between the two domains and arriving at a higher accuracy on the most difficult domain (and also on average). Results are further improved by applying a final supervised  fine-tuning step to our model without discriminator, and with classification loss re-weighted depending on the domain classification error. Note that this fine-tuning step, while improving classification, hampers the domain alignment and therefore the reusability of features for alternative tasks as well as the domain privacy. More details in Supplementary Note C.1.

\smallskip

Table \ref{table:officecaltech} displays the average results of our bi-domain experiments on the Office-Caltech dataset. We observe that the joint critic (ours) is always better than the marginal one. We also observe that  the joint critic, compared to the lack thereof, leads to more uniform results across domains (and therefore higher minimum accuracy), as well as higher average. These observations are consistent with our theoretical results.

\begin{table}[h]
    \centering
    \begin{tabular}{l|cccc|cc}
        \toprule
        & \textit{Amazon} &
        \textit{Caltech} &
        \textit{DSLR} &
        \textit{Webcam}& Avg & Min\\\midrule
        No critic   & \textit{90.63} & \textit{84.27} & \textit{93.75} & \textit{98.31} & 91.74 & 84.27\\
        Marginal critic  & \textit{91.32} & \textit{85.46} & \textit{92.71} & \textit{96.07} & 91.39 & 85.46\\
        Joint critic (Ours)  & \textit{91.67} & \textit{86.05} & \textit{93.75} &\textit{97.00} & \textbf{92.12} & \textbf{86.05}\\
        \bottomrule
    \end{tabular}
    \caption{Comparison of our method to a classic marginal domain critic and an absence of critic, on the Office-Caltech dataset. We use 200 labels per domain except for DSLR which uses 100. Accuracy is reported as averages over of all bi-domain tasks. Best results overall are in bold.
    }
    \label{table:officecaltech}
\end{table}

Finally, Table \ref{table:pacs} shows prediction performance on the more complex PACS dataset. We test our model on this data in a one-vs-rest setting, so that the model must learn to be invariant between one domain and the three remaining domains.

\begin{table}[h]
    \centering
    \begin{tabular}{l|cccc|cc}
        \toprule
        & \textit{Art} &
        \textit{Cartoon} &
        \textit{Photo} &
        \textit{Sketch}& Avg & Min\\
        & \textit{vs.\ R} & \textit{vs.\ R}  & \textit{vs.\ R}  & \textit{vs.\ R}  \\\midrule
        No critic   & \textit{84.03} & \textit{85.62} & \textit{78.74} & \textit{60.45} & 77.21 & 60.45\\
        Marginal critic  & \textit{84.08} & \textit{87.07} & \textit{78.26} & \textit{64.44} & 78.46 & 64.44\\
        Joint critic (Ours)  & \textit{77.15} & \textit{88.61} & \textit{83.41} &\textit{71.52} & \textbf{80.18} & \textbf{71.52}\\
        \bottomrule
    \end{tabular}
    \caption{Comparison of our method to a classic marginal domain critic and an absence of critic, on the PACS dataset. We use 500 labels per domain. Accuracy is reported as Domain vs.\ Rest. Best results overall are in bold.}
    \label{table:pacs}
\end{table}

Again, we find that our model produces the best minimum and average accuracy in each scenario. We found that a trade-off may exist between Art and other domains. Although our method performs worse than competitors on this domain, we observe that it leads to domain accuracies more concentrated around the mean, and therefore a higher minimum accuracy. Additionally, we note that the average accuracy has also increased. 

\smallskip

Lastly, we would like to reiterate that the problem of domain invariance has received considerably less attention in the context of deep neural networks than the tasks of domain adaptation and domain generalization. Our quantitative results as well as the multiple baseline results aim to provide useful reference values for future work on domain invariance.

\subsection{Visual Insights on Learned Representations}

While results in the section above have verified quantitatively the performance of our proposed domain invariant network, we would like to also present some qualitative insights.

As a first experiment, we visualize how the representation of the Conv-Large model trained with our proposed approach becomes more task-specific and less domain-dependent throughout training. For this, we take samples from $\mathcal{P}_\mathcal{Z}$ and $\mathcal{Q}_\mathcal{Z}$, join them, and perform a low-dimensional embedding of the resulting distribution via UMAP \cite{mcinnes2018umap}. Plots before and after training are shown in Figure \ref{fig:insights} (left). The visualization suggests that the two domains are strongly separated initially, but under the influence of domain invariant training, they collapse to the same regions in representation space. As expected, the learned representation also better resolves the different classes after training (here roughly given by the cluster structure).

\begin{figure}[h]
    \centering
    \includegraphics[width=.98\textwidth]{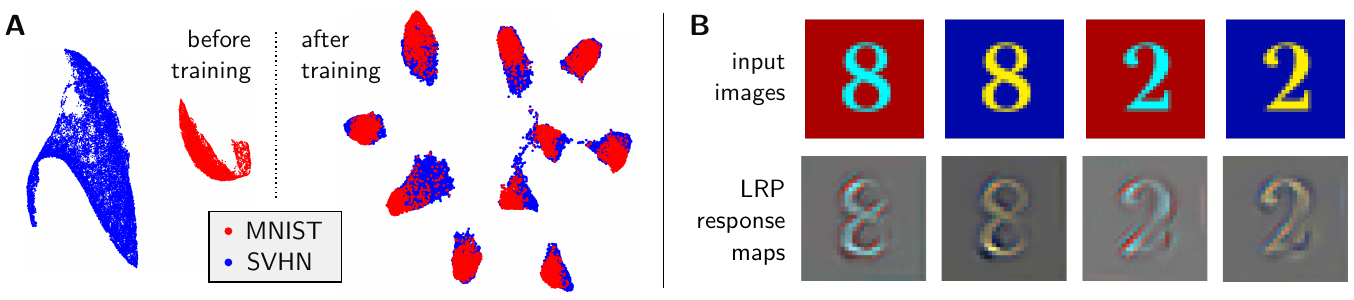}
    \caption{Left: UMAP visualization of the extracted representation before and after training. Right: Response (extracted using LRP) of the model to various input digits with different style (color).}
    \label{fig:insights}
\end{figure}

As a second experiment, we present SVHN-like synthetic examples to our domain invariant network and vary the digit and the colors. Using the Layer-wise Relevance Propagation (LRP) explanation method \cite{bach-plos15}, we then compute for each prediction the local response of the model. The LRP method identifies the contribution of each input pixel to the prediction. These pixel-wise contributions can also be seen as the summands of a linear model, and the latter forms a local interpretable surrogate for the original model. We refer to weights of this linear model as the `LRP response map' (details on LRP and how to generate response maps are given in Supplementary Note C.2).

A selection of examples and associated LRP response maps is shown in Figure \ref{fig:insights} (right), featuring two digit classes and SVHN-like color variations. Although we would expect that style and color play a marginal role in representation space (our objective has enforced invariance between the colored SVHN and the black\&white handwritten MNIST domains), recognizing such style and color variations remains an integral part of the neural network prediction strategy. We indeed observe that the model precisely adapts to the input digit by providing {\em domain-specific} response maps of corresponding colors. This strategy is therefore instrumental in the process of building the domain invariant representation.

\section{Conclusion}
Real-world data is often heterogeneous, subject to sub-population shifts, or coming from multiple domains. In this work, we have for the first time studied the problem of learning domain-invariant representations as measured by the joint Wasserstein distance. We have created a theoretical framework for semi-supervised domain invariance and have contributed several upper-bounds to the Wasserstein distance of joint distributions that links domain invariance to practical learning objectives.

In our benchmark experiments, we find that optimizing the resulting objective leads to high prediction accuracy on both domains while simultaneously achieving high domain invariance, which we also observe qualitatively on low-dimensional embedding visualizations. We have further observed, somewhat counterintuitively, that domain adversarial training can still result in a model that makes use of domain-specific features in order to arrive at the domain-invariant representations.

Our work allows for several future extensions. For example, it would be interesting to obtain a theoretical connection to other representation learning methods, in particular, contrastive learning, that may be integrated to our framework. Furthermore, an extension of our theory to domain generalization could enable further applications and increase our understanding of domain generalization itself.

Overall, our work on domain invariance provides new theoretical insights as well as quantitative competitive results for a number of scenarios and baselines. We believe it thereby constitutes a useful first basis for further research on domain-invariant ML models and applications thereof.

\subsubsection*{Acknowledgements}

KRM was partly supported by the Institute of Information \& Communications Technology Planning \& Evaluation (IITP) grants funded by the Korea government(MSIT) (No. 2019-0-00079, Artificial Intelligence Graduate School Program, Korea University and No. 2022-0-00984, Development of Artificial Intelligence Technology for Personalized Plug-and-Play Explanation and Verification of Explanation). This work was supported in part by the German Ministry for Education and Research (BMBF) under Grants 01IS14013A-E, 01GQ1115, 01GQ0850, 01IS18025A and 01IS18037A; and by JSPS KAKENHI Grant Number 20H00576, and 19H04071. Correspondence to TK, KRM and GM. 

{\small
\bibliographystyle{abbrv}
\bibliography{references}
}

\end{document}


\begin{frontmatter}

\title{Learning Domain Invariant Representations by Joint Wasserstein Distance Minimization\\[1mm]\textsc{\large (Supplementary Notes)}}
  
\author{L\'eo~And\'eol}
\author{Yusei~Kawakami}
\author{Yuichiro~Wada}
\author{Takafumi~Kanamori}
\author{\\Klaus-Robert~M\"uller}
\author{Gr\'egoire~Montavon}

\end{frontmatter}

\vspace{-1.25cm}

\section{Proofs of Main Results}
\label{appendix:theory}
In the following, we give the proofs of the main theoretical results presented in the paper. After providing the formal mathematical proof, we detail for each of them in the paragraph below the steps taken to reach the final result.
\begin{lemma}
\label{lemma:wasserstein_mixture_decomposition}
Let $\mathcal{K}\in\mathcal{M}_{+}^1(\mathcal{Z}\times\mathcal{Y})$ be an arbitrary probability distribution, we then have
$ W(\mathcal{P},\mathcal{K})\leq \alpha W(\mathcal{P}^t, \mathcal{K}) + (1-\alpha)W(\mathcal{P}^f, \mathcal{K}),
$ and for the special case where $\mathcal{K}=\mathcal{P}^t$, we get
\begin{equation}
    W(\mathcal{P},\mathcal{P}^t) = (1-\alpha)W(\mathcal{P}^f, \mathcal{P}^t).
    \label{eq:lemma1-specific}
\end{equation}
\end{lemma}
\begin{proof}
\begin{align}
    \label{lemma1a} W(\mathcal{P},\mathcal{K})&=
    \sup\limits_{\norm{\varphi}_{Lip}\leq 1} \left\{\int_\mathcal{Z} \varphi d\mathcal{P} - \int_\mathcal{Z} \varphi d\mathcal{K} \right\}\\
    \label{lemma1b} &=\sup\limits_{\norm{\varphi}_{Lip}\leq 1} \left\{ \alpha\int_\mathcal{Z} \varphi d\mathcal{P}^t  + (1-\alpha)\int_\mathcal{Z} \varphi d\mathcal{P}^f - \int_\mathcal{Z} \varphi d\mathcal{K} \right\}\\
    \label{lemma1c}&= \sup\limits_{\norm{\varphi}_{Lip}\leq 1} \left\{ \alpha\left(\int_\mathcal{Z} \varphi d\mathcal{P}^t - \int_\mathcal{Z} \varphi d\mathcal{K}\right) + (1-\alpha)\left(\int_\mathcal{Z} \varphi d\mathcal{P}^f - \int_\mathcal{Z} \varphi d\mathcal{K}\right) \right\}\\
    \label{lemma1d} &\leq \sup\limits_{\norm{\varphi}_{Lip}\leq 1} \left\{ \alpha\left(\int_\mathcal{Z} \varphi d\mathcal{P}^t - \int_\mathcal{Z} \varphi d\mathcal{K}\right) \right\}\\
    &\qquad+ \sup\limits_{\norm{\varphi}_{Lip}\leq 1} \left\{ (1-\alpha)\left(\int_\mathcal{Z} \varphi d\mathcal{P}^f - \int_\mathcal{Z} \varphi d\mathcal{K}\right) \right\}\\
    \label{lemma1e} &= \alpha\sup\limits_{\norm{\varphi}_{Lip}\leq 1} \left\{ \int_\mathcal{Z} \varphi d\mathcal{P}^t - \int_\mathcal{Z} \varphi d\mathcal{K} \right\}+ (1-\alpha)\sup\limits_{\norm{\varphi}_{Lip}\leq 1} \left\{ \int_\mathcal{Z} \varphi d\mathcal{P}^f - \int_\mathcal{Z} \varphi d\mathcal{K} \right\}\\
    \label{lemma1f} &= \alpha W(\mathcal{P}^t, \mathcal{K}) + (1-\alpha)W(\mathcal{P}^f, \mathcal{K})
\end{align}
We start the proof in Eq.\ \eqref{lemma1a} by simply stating the definition of the dual of the 1-Wasserstein distance (with the cost function being the metric of our space). Because $\alpha \in (0,1)$ and $\mathcal{P}^t$ and $\mathcal{P}^f$ are finite, we can decompose the integral of a mixture of measures as a mixture of integrals of each of the elementary measures, as we do in Eq.\ \eqref{lemma1b}. In Eq.\ \eqref{lemma1c} we also decompose the second integral into parts weighted by the same $\alpha$ and $1-\alpha$ (which sum to one), group those with the corresponding integral on the other measure, and factorize by their weights. Using a property of the supremum in Eq.\ \eqref{lemma1d}, we upper bound the supremum of a sum by the sum of the supremum. What we obtain then by pulling out constant in Eq.\ \eqref{lemma1e} is a sum of two Wasserstein distances by their dual definition, and we hence obtain Eq.\ \eqref{lemma1f} and complete the proof.
For the second result, in the case where $\mathcal{K}=\mathcal{P}^t$, the first member of Eq.\ \eqref{lemma1d} cancels out, therefore we do not need to apply the supremum bound, and therefore we have an equality.
 \end{proof}
\begin{lemma}
\label{lemma:mse_bound} 
\begin{equation}
    \begin{aligned}
W(\mathcal{P}^{t}, \mathcal{P}^{f}) &\leq\mathbb{E}_{(z,y) \sim \mathcal{P}^t}\big[|y-f(z)|\big]
\end{aligned}
\end{equation}
\end{lemma}
\begin{proof}
\begin{equation}
\begin{aligned}
W(\mathcal{P}^{t}, \mathcal{P}^{f}) &=\inf _{\pi} \int c((z, y),(z^{\prime}, f(z^{\prime})) \mathrm{d} \pi((z, y),(z^{\prime}, f(z^{\prime}))\\
& \leq \int c((z, y),(z, f(z)) \mathrm{d} \mathcal{P}^{t}(z, y)\\
&=\mathbb{E}[|y-f(z)|]
\end{aligned}
\end{equation}
\end{proof}
\begin{lemma}
\label{lemma:true_estimated_dist} 
Assuming that $\mathcal{P}^f$ and $\mathcal{P}^t$ admit densities, we then obtain
\begin{equation}
    W(\mathcal{P}^t, \mathcal{P}^f)\leq \operatorname{diam}(\mathcal{Z}\times\mathcal{Y})\sqrt{\frac{1}{2}\mathbb{E}_{z\sim\mathcal{P}_\mathcal{Z}}\left[\KL{\mathcal{P}^t_{\mathcal{Y}|z}}{\mathcal{P}^f_{\mathcal{Y}|z}}\right]},\\
\end{equation}
where $\KL{\mathcal{P}}{\mathcal{Q}}=-\int_{\mathcal{Z}\times\mathcal{Y}} \log \frac{d \mathcal{Q}}{d \mathcal{P}} d \mathcal{P}$ is the Kullback-Leibler divergence; and where $\operatorname{diam}(\cdot)$ is the diameter of the space received as input, i.e.\ the largest distance obtainable in that space.
\end{lemma}
 
\begin{proof}
In order to prove this result we have to rely on an upper bound of the Wasserstein distance by the Kullback-Leibler divergence, through the combination of two standard bounds. We therefore present this result here and a quick proof.
\begin{center}
\fbox{~~
\begin{minipage}{.85\textwidth}
\begin{unnumlemma}{\textbf{From \cite{gibbs2002choosing}}}
Let $\mathcal{P}, \mathcal{Q}\in\mathcal{M}_{+}^1(\mathcal{Z}\times\mathcal{Y})$ be two probability distributions on a compact measurable space $\mathcal{Z}\times\mathcal{Y}$, we then have
\begin{equation}
    W(\mathcal{P}, \mathcal{Q})\leq\operatorname{diam}(\mathcal{Z}\times\mathcal{Y})\sqrt{\frac{1}{2}\KL{\mathcal{P}}{\mathcal{Q}}}
\end{equation}
\centering and
\begin{equation}
    W(\mathcal{P}, \mathcal{Q})\leq\operatorname{diam}(\mathcal{Z}\times\mathcal{Y})\sqrt{\frac{1}{2}\KL{\mathcal{Q}}{\mathcal{P}}}
\end{equation}
\end{unnumlemma}
\begin{proof}
Combine the bound of the Wasserstein distance by the Total Variation distance (Theorem 4 of \cite{gibbs2002choosing}), and that one by the Kullback-Leibler Divergence using Pinsker's Inequality.
\end{proof}
\end{minipage}~~}
\end{center}
With that result, we show that under our conditions, the Kullback-Leibler divergence on marginals is in fact the expected KL divergence (on the marginal distribution $\mathcal{P}_\mathcal{Z}$) on the conditional distribution. Let $\rho^t, \rho^f$ be the densities of respectively $\mathcal{P}^t, \mathcal{P}^f$.
\begin{align}
    \label{lemma2a} \KL{\mathcal{P}^t}{\mathcal{P}^f}
    &= -\int_{\mathcal{Z}\times\mathcal{Y}} \rho^t_{\mathcal{Z},\mathcal{Y}}(z,y) \operatorname{log} \frac{\rho^f_{\mathcal{Z},\mathcal{Y}}(z,y)}{\rho^t_{\mathcal{Z},\mathcal{Y}}(z,y)}dydz\\
    \label{lemma2b} &= -\int_{\mathcal{Z}}\rho^t_{\mathcal{Z}}(z)\sum\limits_{\mathcal{Y}} \mathcal{P}^t(y|z) \operatorname{log} \frac{\rho^f_{\mathcal{Z}}(z)\mathcal{P}^f(y|z)}{\rho^t_{\mathcal{Z}}(z)\mathcal{P}^t(y|z)}dz\\
    \label{lemma2c} &= \mathbb{E}_{z\sim\mathcal{P}_\mathcal{Z}}\left[-\sum\limits_{\mathcal{Y}} \mathcal{P}^t(y|z) \operatorname{log} \frac{\rho^f_{\mathcal{Z}}(z)\mathcal{P}^f(y|z)}{\rho^t_{\mathcal{Z}}(z)\mathcal{P}^t(y|z)}\right]\\
    \label{lemma2d} &= \mathbb{E}_{z\sim\mathcal{P}_\mathcal{Z}}\left[-\sum\limits_{\mathcal{Y}} \mathcal{P}^t(y|z) \operatorname{log} \frac{\mathcal{P}^f(y|z)}{\mathcal{P}^t(y|z)}\right]\\
    \label{lemma2e} &= \mathbb{E}_{z\sim\mathcal{P}_\mathcal{Z}}\left[\KL{\mathcal{P}^t_{\mathcal{Y}|z}}{\mathcal{P}^f_{\mathcal{Y}|z}}\right]
\end{align}
The first line (Eq.\ \eqref{lemma2a}) is the definition of the Kullback-Leibler divergence with densities. Eq.\ \eqref{lemma2b} is an application of Fubini's theorem which allows us to decompose the double integral and a decomposition of joint probability into the product of marginal and conditional. Finally as $\mathcal{Y}$ is a discrete space, the integral becomes a sum of probabilities, and $\rho^t_\mathcal{Z}(z)$ is pulled out of the sum. Eq.\ \eqref{lemma2c} replaces the integral by the expectation, by definition. In Eq.\ \eqref{lemma2d}, since by definition $\rho^t_{\mathcal{Z}}(z)=\rho^f_{\mathcal{Z}}(z)$, those terms are removed from the fraction. Eq.\ \eqref{lemma2e} is again an application of the definition of the KL divergence.\\
By combining the equality obtained in Eq.\ \eqref{lemma2e} and the cited lemma, we complete the proof.
\end{proof}
\begin{theorem}
\label{th_final}
Given the Wasserstein distance's cost function $c$ is the metric on the product space $\mathcal{Z}\times\mathcal{Y}$, we get
\begin{equation}
    W(\mathcal{P}^t, \mathcal{Q}^t) \leq  (1-\alpha)W(\mathcal{P}^t,\mathcal{P}^f)
    + W(\mathcal{P},\mathcal{Q}) +  (1-\beta)W(\mathcal{Q}^t,\mathcal{Q}^f).
    \label{eq:theorem}
\end{equation}
\end{theorem}
\begin{proof}
\begin{align}
    \label{theorem1a} W(\mathcal{P}^t, \mathcal{Q}^t) & \leq W(\mathcal{P}^t, \mathcal{P}) + W(\mathcal{P}, \mathcal{Q}) + W(\mathcal{Q}, \mathcal{Q}^t)\\
    \label{theorem1c} & = (1-\alpha) W(\mathcal{P}^t, \mathcal{P}^f) + W(\mathcal{P}, \mathcal{Q}) + (1-\beta) W(\mathcal{Q}^t, \mathcal{Q}^f)
\end{align}
Eq.\ \eqref{theorem1a} is an application of the triangle inequality. Using the symmetry of the metric (and therefore of the Wasserstein distance), and applying Lemma \ref{lemma:wasserstein_mixture_decomposition} twice, we obtain Eq.\ \eqref{theorem1c}. 
\end{proof}
\begin{theorem}
\label{th_diff_err}
Let $\mathcal{Z},\mathcal{Y}$ be two compact measurable metric spaces whose product space has dimension $d$. Let $\mathcal{P}^t,\mathcal{Q}^t\in\mathcal{M}_{+}^{1}(\mathcal{Z}\times\mathcal{Y})$ two joint distributions associated to the two domains, and $\widehat{\mathcal{P}}^t,\widehat{\mathcal{Q}}^t$ their empirical counterparts. Let the transport cost function $c$ associated to the optimal transport problem be $c(z_1,y_1;z_2,y_2)=\norm{\veccol{z_1 \\ y_1}-\veccol{z_2 \\ y_2}}_2$, the Euclidean distance as the metric on $\mathcal{Z}\times\mathcal{Y}$ and $\mathcal{L}:\mathcal{Y}\times\mathcal{Y}\rightarrow\R^{+}$ a symmetric $\kappa$-Lipschitz loss function. Then for any $d^\prime > d$ and $\psi^\prime<\sqrt{2}$ there exists some constant $N_0$ depending on $d^\prime$ such that for any $\delta > 0$ and $\min(N_P, N_Q)\geq N_0\max(\delta^{-(d^\prime+2)},1)$ with probability at least $1-\delta$ for all $\lambda$-Lipschitz $f$ the following holds:
\begin{equation}
    \abs{\mathcal{R}_{\mathcal{Q}^t}(f)-\mathcal{R}_{\mathcal{P}^t}(f)} \leq \kappa\sqrt{(\lambda^2 +1)}\left[W_{1}(\widehat{\mathcal{P}}^t,\widehat{\mathcal{Q}}^t)+\sqrt{\frac{2}{\psi^{\prime}}\log \left(\frac{1}{\delta}\right)}\left(\sqrt{\frac{1}{N_{P}}}+\sqrt{\frac{1}{N_{Q}}}\right)\right].
\end{equation}
\label{theorem:gap}
\end{theorem}
\begin{proof}
Let $\pi^\ast\in\Pi(\mathcal{P}^t, \mathcal{Q}^t)$ be the optimal coupling. 
\begin{align}
    \label{theorem2a} \abs{\mathcal{R}_{\mathcal{Q}^t}(f)-\mathcal{R}_{\mathcal{P}^t}(f)} &= \abs{\mathbb{E}_{z,y\sim \mathcal{Q}^t}[\mathcal{L}(f(z), y)]-\mathbb{E}_{z,y\sim \mathcal{P}^t}[\mathcal{L}(f(z), y)]}\\
    \label{theorem2b} &= \abs{\int_{\mathcal{Z}\times\mathcal{Y}}\mathcal{L}(f(z),y)d\mathcal{Q}^t(z,y)-\int_{\mathcal{Z}\times\mathcal{Y}}\mathcal{L}(f(z),y)d\mathcal{P}^t(z,y)}\\
    \label{theorem2c} &= \abs{\int_{\mathcal{Z}\times\mathcal{Y}}\mathcal{L}(f(z),y)d(\mathcal{Q}^t-\mathcal{P}^t)(z,y)}\\
    \label{theorem2d} &= \abs{\int_{(\mathcal{Z}\times\mathcal{Y})^2}\mathcal{L}(f(z),y)-\mathcal{L}(f(z^{\prime}),y^{\prime})d\pi^{*}((z,y),(z^{\prime},y^{\prime}))}\\
    \label{theorem2e} &\leq \int_{(\mathcal{Z}\times\mathcal{Y})^2}\abs{\mathcal{L}(f(z),y)-\mathcal{L}(f(z^{\prime}),y^{\prime})}d\pi^{*}((z,y),(z^{\prime},y^{\prime}))\\
    \label{theorem2f} &= \int_{(\mathcal{Z}\times\mathcal{Y})^2}\abs{\mathcal{L}(f(z),y)-\mathcal{L}(f(z^{\prime}),y)\\&\qquad\qquad+\mathcal{L}(f(z^{\prime}),y)-\mathcal{L}(f(z^{\prime}),y^{\prime})}d\pi^{*}((z,y),(z^{\prime},y^{\prime}))\\
    \label{theorem2g} &\leq \int_{(\mathcal{Z}\times\mathcal{Y})^2}\abs{\mathcal{L}(f(z),y)-\mathcal{L}(f(z^{\prime}),y)}\\&\qquad\qquad+\abs{\mathcal{L}(f(z^{\prime}),y)-\mathcal{L}(f(z^{\prime}),y^{\prime})}d\pi^{*}((z,y),(z^{\prime},y^{\prime}))\\
    \label{theorem2h} &\leq \int_{(\mathcal{Z}\times\mathcal{Y})^2}\kappa\norm{f(z)-f(z^{\prime})}_2+\kappa\norm{y-y^{\prime}}_2 d\pi^{*}((z,y),(z^{\prime},y^{\prime}))\\
    \label{theorem2i} &\leq \int_{(\mathcal{Z}\times\mathcal{Y})^2}\kappa\lambda \norm{z-z^{\prime}}_2+\kappa\norm{y-y^{\prime}}_2 d\pi^{*}((z,y),(z^{\prime},y^{\prime}))\\ 
    \label{theorem2j} &\leq \int_{(\mathcal{Z}\times\mathcal{Y})^2} \kappa \sqrt{(\lambda^2 + 1)} \norm{\veccol{z \\ y}-\veccol{z^{\prime} \\ y^{\prime}}}_2 d\pi^{*}((z,y),(z^{\prime},y^{\prime}))\\
    \label{theorem2k} &= \kappa \sqrt{(\lambda^2 + 1)} \int_{(\mathcal{Z}\times\mathcal{Y})^2} \norm{\veccol{z \\ y}-\veccol{z^{\prime} \\ y^{\prime}}}_2 d\pi^{*}((z,y),(z^{\prime},y^{\prime}))\\
    \label{theorem2l} &= \kappa \sqrt{(\lambda^2 + 1)} W_1 (\mathcal{P}^t,\mathcal{Q}^t)
\end{align}
We start the demonstration by replacing definition by explicit formulations, in Eq.\ \eqref{theorem2a} and again in Eq.\ \eqref{theorem2b}. In Eq.\ \eqref{theorem2c} we replace a difference of integrals by the integral of the difference of measures, which leads to a form related to the dual of the Wasserstein distance. A consequence of the Kantorovich-Rubinstein duality theorem is that Eq.\ \eqref{theorem2c} and \eqref{theorem2d} are equal for the optimal coupling. The next equation, Eq.\ \eqref{theorem2e} is a property of the absolute value, namely $\abs{a+b}\leq\abs{a}+\abs{b}$. We then add two terms summing to zero in Eq.\ \eqref{theorem2f} and apply the same property of the absolute value again, to obtain Eq.\ \eqref{theorem2g}. Having the absolute value of the difference of a Lipschitz function for two values allows up to upper bound that difference by one on the inputs of the function, up to a Lipschitzness factor. We apply that operation on two terms to obtain Eq.\ \eqref{theorem2h} and again on the first term to obtain Eq.\ \eqref{theorem2i}. Using the Cauchy-Schwartz inequality, the sum of two Euclidean distances can be upper bounded by the Euclidean distance between the concatenated vectors as in Eq.\ \eqref{theorem2j}, which corresponds to the cost function used in our discriminator throughout the main paper. The next steps correspond to pulling out constant outside of the integral (Eq.\ \eqref{theorem2k}), and replacing the explicit formulation of the 1-Wasserstein distance by its notation.

We have completed the main part of the proof. The next step is to apply a classical concentration bound which allows us to replace the distributions in the Wasserstein distance by their empirical counterparts. We now reintroduce the concentration bound we use:
\begin{center}
\fbox{~~
\begin{minipage}{.85\textwidth}
\begin{unnumtheorem}{\textbf{From \cite{bolley2007quantitative}
}}
Let $\mu$ be a probability measure on $\R^d$ satisfying a $T_1(\psi)$ inequality and let $\widehat{\mu}^{N}:=\frac{1}{N} \sum_{i=1}^{N} \delta_{X^{i}}$ be its associated empirical measure. Then, for any $d^{\prime}>d$ and $\psi^{\prime}< \psi$, there exists some constant $N_0$, depending on $\psi^{\prime}$, $d^{\prime}$ and some square-exponential moment of $\mu$, such that for any $\epsilon>0$ and $N\geq N_0 \max(\epsilon^{-(d^{\prime}+2)},1)$,
\begin{equation}
\mathbb{P}\left[W_{1}\left(\mu, \widehat{\mu}^{N}\right)>\epsilon\right] \leq e^{-\frac{\psi^{\prime}}{2} N \epsilon^{2}}
\end{equation}
\end{unnumtheorem}
\end{minipage}~~}
\end{center}
Now by using the triangle inequality, we can relate Wasserstein distances between empirical and true distributions,
\begin{equation}
    \kappa \sqrt{(\lambda^2 + 1)} W_1 (\mathcal{P}^t,\mathcal{Q}^t) \leq \kappa \sqrt{(\lambda^2 + 1)} \left( W_1 (\widehat{\mathcal{P}}^t,\mathcal{P}^t) + W_1 (\widehat{\mathcal{P}}^t,\widehat{\mathcal{Q}}^t) + W_1 (\mathcal{Q}^t,\widehat{\mathcal{Q}}^t) \right),
\end{equation}
and by applying the concentration bound twice, we obtain our final result and complete the proof:
\begin{equation}
    \abs{\mathcal{R}_{\mathcal{Q}^t}(f)-\mathcal{R}_{\mathcal{P}^t}(f)} \leq \kappa\sqrt{(\lambda^2 +1)}\left[W_{1}(\widehat{\mathcal{P}}^t,\widehat{\mathcal{Q}}^t)+\sqrt{\frac{2}{\psi^{\prime}}\log \left(\frac{1}{\delta}\right)}\left(\sqrt{\frac{1}{N_{P}}}+\sqrt{\frac{1}{N_{Q}}}\right)\right].
\end{equation}
\end{proof}

\section{Details of the Experiments of Section 6.2}
\label{appendix:experiments}

\paragraph{Hardware \& Computation} All experiments but PACS were conducted on a single RTX 2060 Super. The PACS experiments were conducted on a single TITAN RTX. All experiments were conducted on a desktop computer. Most experiments lasted between 1 and 3 hours and none more than 6 hours.

\paragraph{Implementation} Our model is implemented using pytorch \cite{NEURIPS20199015} and torchvision as framework, timm \cite{rw2019timm} and nfnets-pytorch \cite{nfnets2021pytorch} for access to normalizer-free networks, PythonOT \cite{flamary2021pot} to compute the Wasserstein distance reported in the tables. Our code is available at \url{https://github.com/leoandeol/ldir}. It contains everything necessary to reproduce experiments to the exception of the data itself, which can be easily obtained from the official sources.

\paragraph{Results of Table 1 of the main paper} We use a Cross-Entropy (equivalent to the Kullback-Leibler divergence in case of a deterministic labeling rule) classification loss with a weight of 1 and no regularization losses, and the domain critic also had a weight of 1. We use the Conv-Large network for SVHN from the VAT paper \citep{miyato2018virtual}, and simple data augmentation (small translations, and color jittering provided with pytorch) for all experiments.

\paragraph{Results of Table 2 of the main paper}
We use VAT with $\epsilon=2.5$ and $\xi=10$. VAT, conditional entropy, and classification losses all had a weight of 1, and the domain critic had a weight of 1. We do not use Virtual MixUp. Fine-Tuning (for classification) consists in one more epoch, without discriminator loss, and with the loss intuitively reweighted by the error of each domain (0.25 for MNIST to not forget, 0.75 for SVHN to improve). We use the Conv-Large network, and simple data augmentation (small translations, and color jittering provided with pytorch) for all experiments.

\paragraph{Results of Table 3 of the main paper} We used Decaf features \cite{donahue2014decaf} with a Resnet18 backbone \cite{he2016deep}. Except for this, the settings are identical as in Table 1.

\paragraph{Results of Table 4 of the main paper}
Unless stated otherwise, we use the same settings as Table 1. We use VAT with an adaptive radius such as introduced in \cite{hu2017learning} with the same parameters. We use a different data augmentation, relying on RandAugment \cite{cubuk2020randaugment} with $n=2$ and $m=5$.

\section{Details of the Analyses of Section 6.3}
\label{appendix:analyses}

In this note, we give details of the implementation of UMAP and LRP analyses performed on the learned representations and classifiers. We also provide further UMAP visualizations.

\subsection{Application of UMAP}

\paragraph{Implementation} To compute the two-dimensional UMAP embeddings of the learned representations, we used the official implementation of UMAP \cite{mcinnes2018umap}, with the Euclidean distance and a number of neighbors of 75. We kept all other parameters as default. 

\paragraph{Observations} In addition to the embeddings shown in the main paper, we show in Figure \ref{fig:representation_visualization} further embeddings corresponding to the networks presented in Table 2 of the main paper. Beyond the obvious improvement over untrained features, we observe that supervised and semi-supervised approaches in (b) and (c), extract class structure (visible as distinct clusters), but tend to not produce strongly domain-invariant representations. Our method (d) incorporates domain alignment in the objective and we observe a much stronger overlap between the red and blue points representing the MNIST and SVHN domains respectively. However, we have shown in the paper that a simple epoch of fine-tuning can lead to higher accuracies of classification but comes with lower domain alignment (e).

\begin{figure*}[ht!]
         \centering
         \parbox{.32\textwidth}{\centering (a) Untrained}~
         \parbox{.32\textwidth}{\centering (b) Supervised on Both}~
         \parbox{.32\textwidth}{\centering (c) Semi-sup\\(VAT and EntMin)}\\
        \includegraphics[width=.32\textwidth]{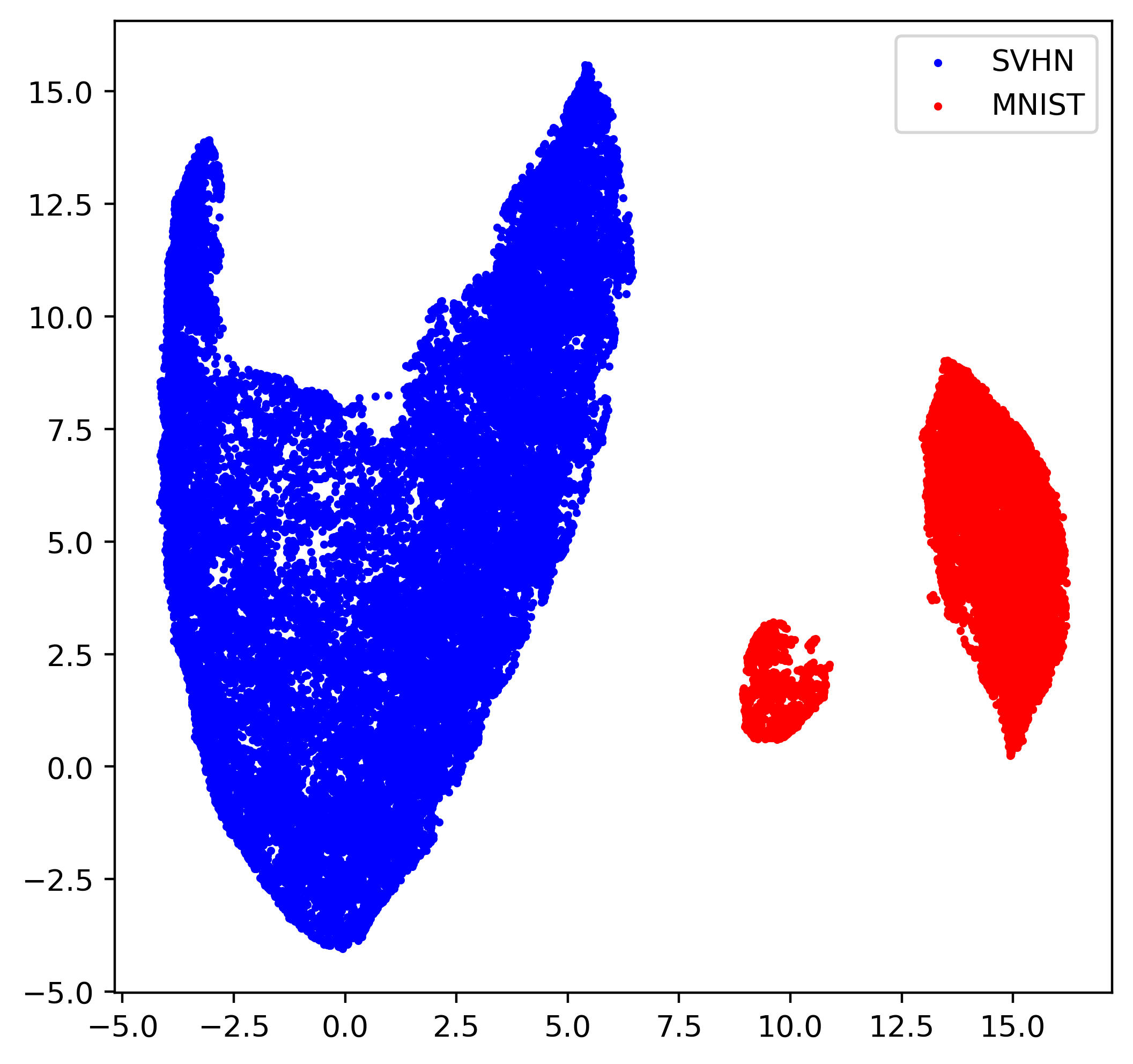}~
        \includegraphics[width=.32\textwidth]{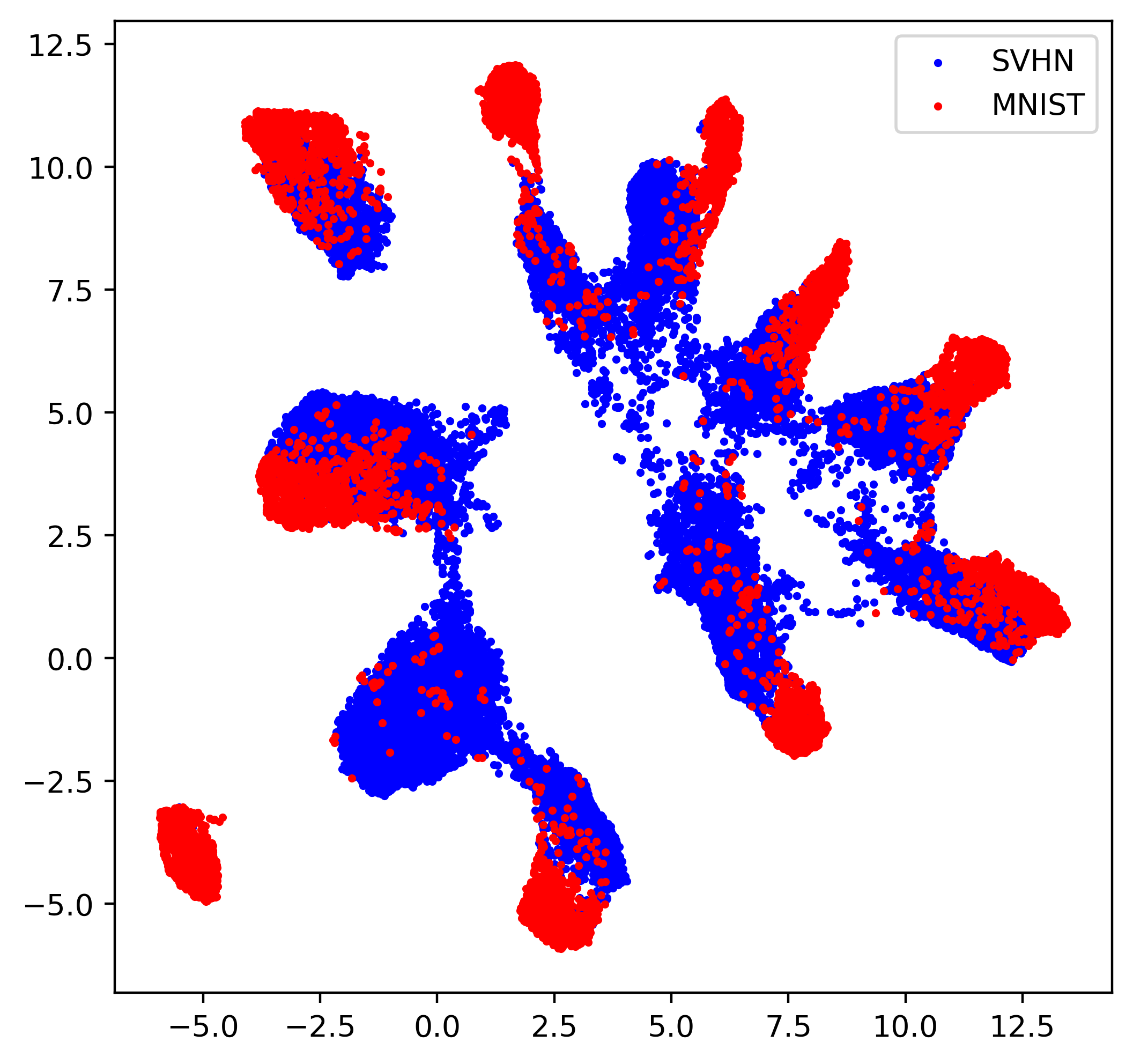}~
        \includegraphics[width=.32\textwidth]{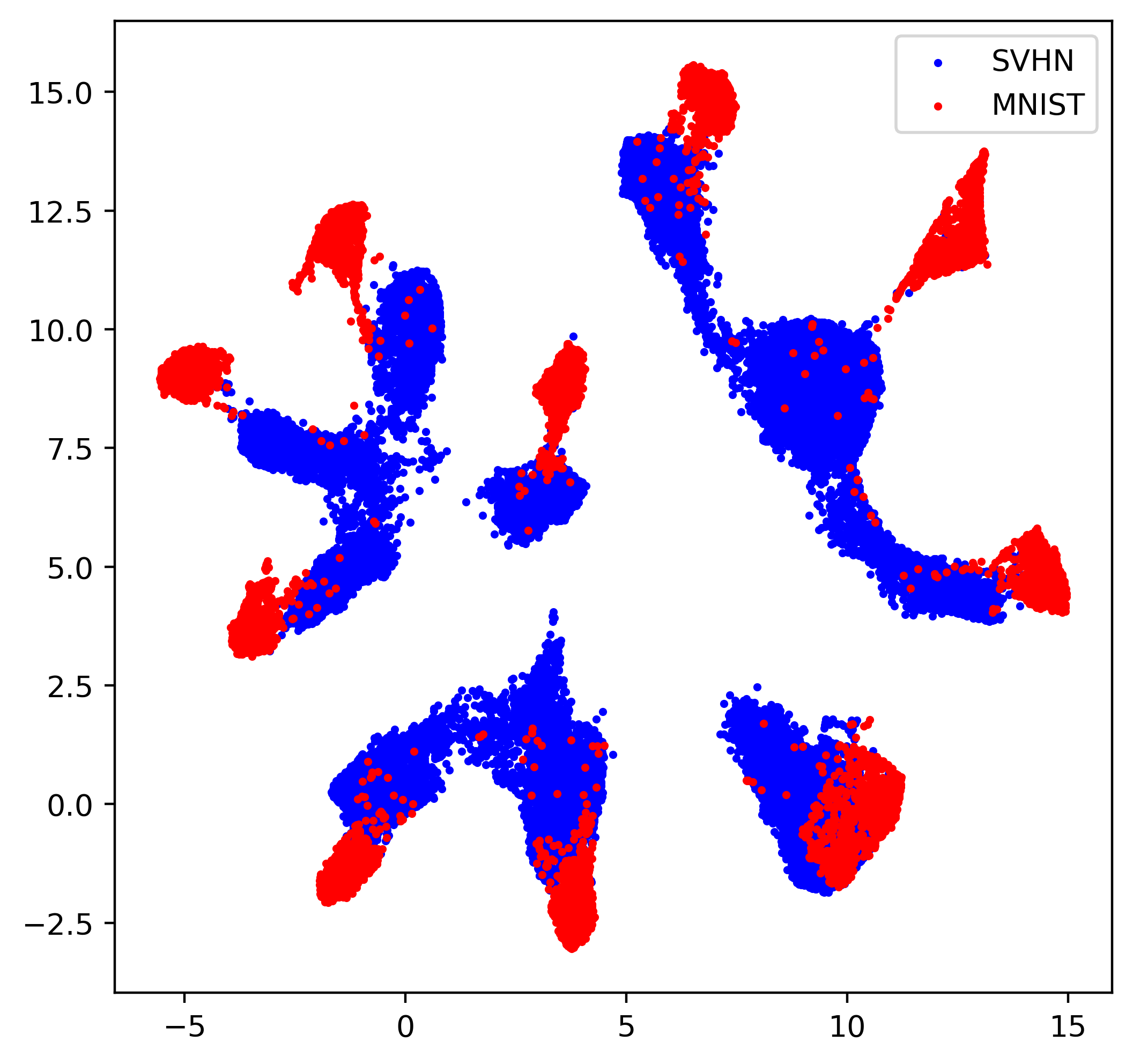}
        
            \medskip
            
            \parbox{.32\textwidth}{\centering (d) Ours (VAT + EntMin)}
            \parbox{.32\textwidth}{\centering (e) Ours (VAT + EntMin)\\+ Fine-Tuning}
            \includegraphics[width=.32\textwidth]{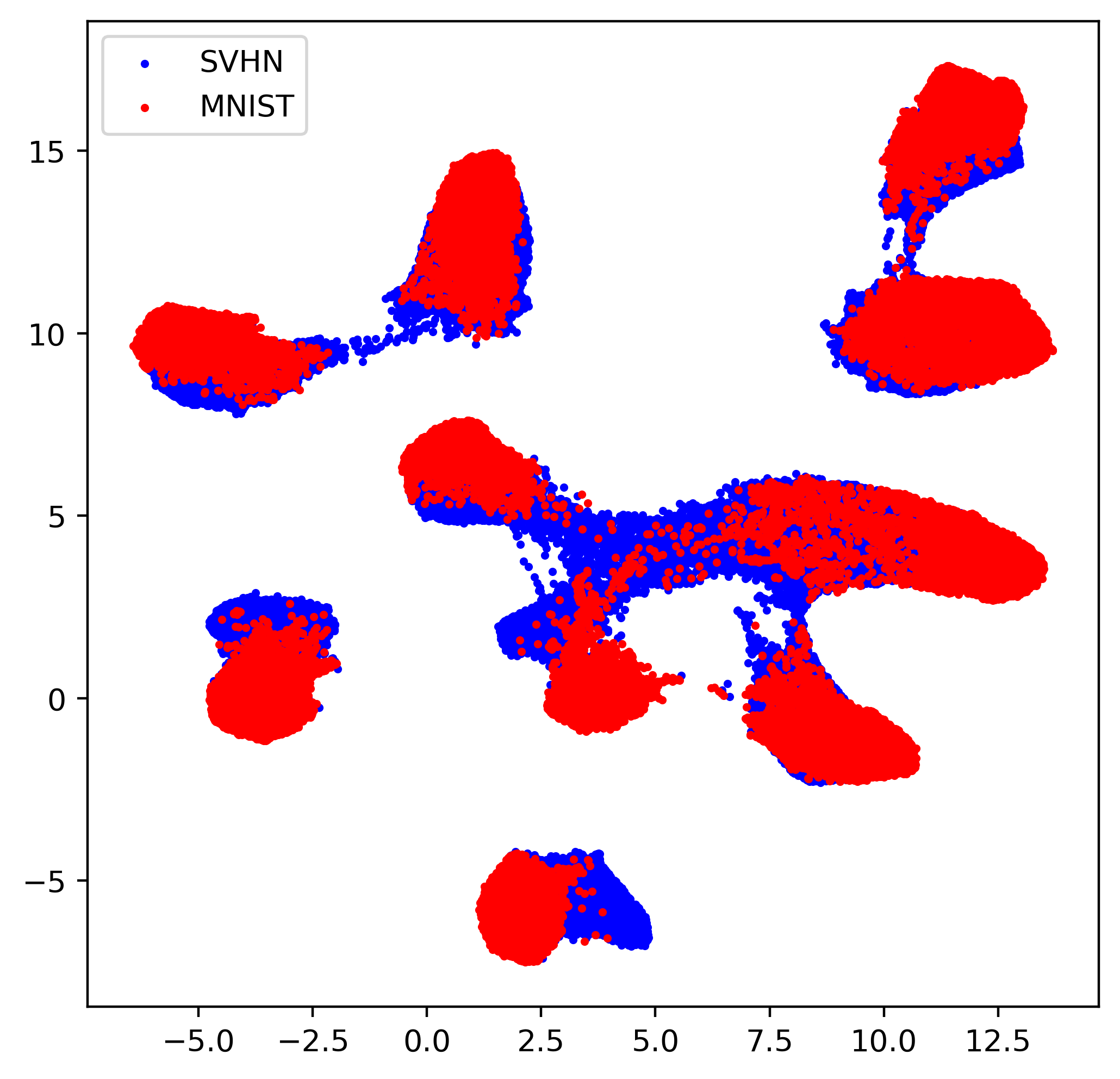}~
            \includegraphics[width=.32\textwidth]{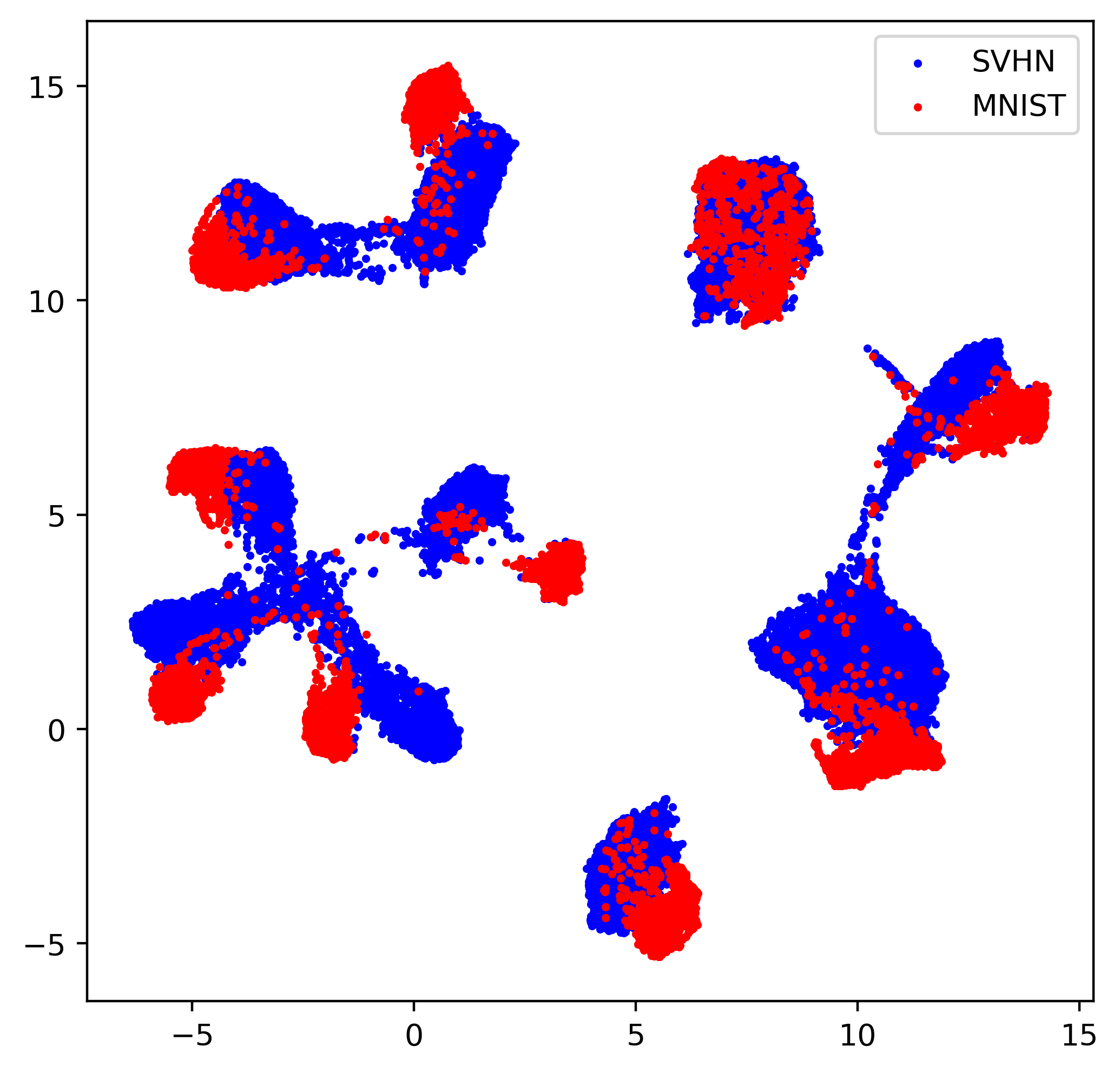}
         
            \caption{Visualization of the representations learned by the neural networks of Table 2 in the main paper. Representations are embedded using UMAP and the different examples are color-coded by domain (SVHN in blue, MNIST in red).}
    \label{fig:representation_visualization}
\end{figure*}

\subsection{Application of LRP}

Layer-wise Relevance Propagation (LRP) \cite{bach-plos15} is a technique to explain the decisions of neural network models in terms of the input features. It can serve to verify the decision strategy employed by a model, to compare two different models, or to visualize the effect of some learning parameters. LRP operates by reverse-propagating the output score, layer after layer, by means of propagation rules. Let $j$ and $k$ be indices for neurons in two consecutive layers of the network, and $\sum_j,\sum_k$ be sum over neurons in the respective layers. A propagation rule typically has the form
$$
R_j = \sum_k \frac{z_{jk}}{\epsilon_k + \sum_{j} z_{jk}} R_k
$$
where $R_j$ and $R_k$ are the `relevance' of neurons $j$ and $k$, where $z_{jk}$ quantifies the contribution of neuron $j$ to the activation of of neuron $k$, and where $\epsilon_k$ is an optional stabilization term. Various rules and heuristics to apply these rules have been proposed \cite{bach-plos15,DBLP:series/lncs/MontavonBLSM19,DBLP:journals/dsp/MontavonSM18}. The LRP procedure terminates once the propagation procedure has reached the input layer. The collection of relevance scores assigned to features in the input layer form the explanation.

When applying LRP to the Conv-Large architecture \cite{miyato2018virtual} considered in our experiments, the specific layers of that architecture need to be addressed by appropriate choice of propagation rules. The Conv-Large architecture is composed of an alternation of convolutions, batch-normalizations, Leaky ReLUs, and max-pooling functions. We follow the strategy proposed in \cite{DBLP:series/lncs/MontavonBLSM19} of fusing batch-normalization layers into the parameters of the adjacent convolution layers, so that we arrive at a simplified but functionally equivalent neural network which consists only of max-pooling layers and convolution-leakyReLU layers. For the max-pooling layers ($a_k = \max\{(a_j)_j\}$) we adopt the commonly used winner-take-all redistribution \cite{bach-plos15}, i.e.\ we redistribute the $R_k$ to the neuron in the pool that has the maximum activation. For the convolution-leakyReLU layers, we extend the LRP-$\gamma$ rule defined in \cite{DBLP:series/lncs/MontavonBLSM19} to account for negative input and output activations. Writing such layers as:
\begin{align*}
z_k &= \textstyle \sum_{0,j} a_j w_{jk}\\
a_k &= z_k \cdot I(z_k>0)\\[1mm]
& \qquad + \alpha z_k \cdot  I(z_k < 0)
\end{align*}
where the convolution is written as a generic weighted sum, where $\sum_{0,j}$ indicates that we sum over all input neurons plus a bias ($b_k = w_{0k}$ with $a_0=1$), and where $\alpha \in [0,1]$ is the leaky ReLU parameter, we define the rule:
\begin{align*}
R_j &= \sum_{k} \frac{a_j^+ \cdot (w_{jk} + \gamma w_{ij}^+)
+ a_j^- \cdot (w_{jk} + \gamma w_{jk}^-)
}{\sum_{0,j} a_j^+ \cdot (w_{jk} + \gamma w_{ij}^+)
+ a_j^- \cdot (w_{jk} + \gamma w_{jk}^-)}\cdot  I(z_k > 0)\cdot  R_k
\\& \qquad + \sum_{k} \frac{a_j^+ \cdot (w_{jk} + \gamma w_{ij}^-) + a_j^- \cdot (w_{jk} + \gamma w_{jk}^+)}{\sum_{0,j} a_j^+ \cdot (w_{jk} + \gamma w_{ij}^-) + a_j^- \cdot (w_{jk} + \gamma w_{jk}^+)} \cdot I(z_k < 0)\cdot  R_k
\end{align*}
where $(\cdot)^+$ and $(\cdot)^-$ are shortcut notations for $\max(0,\cdot)$ and $\min(0,\cdot)$. Like for the original LRP-$\gamma$ rule in \cite{DBLP:series/lncs/MontavonBLSM19}, this expanded LRP rule prioritizes input contributions that agree with neuron output, and $\gamma$ controls the degree of prioritization. Choosing the parameter $\gamma=0$ makes the LRP procedure equivalent to the simple Gradient$\,\times\,$Input method. Choosing $\gamma>0$ adds robustness to the explanation, and such robustness is especially needed in the first layers. Also, for the standard ReLU case ($\alpha=0$ at each layer), all activations become positive, and the proposed LRP rule reduces to the original LRP-$\gamma$ rule. In our analysis, we choose the parameter $\gamma = 0.25$ for the layers 1--17 and the parameter $\gamma= 0$ for the layers 18--35. Lastly, LRP response maps can be obtained by rewriting the relevance scores $(R_i)_i$ obtained in the input layer as $R_i = x_i c_i$ and returning the vector $(c_i)_i$ instead of the usual relevance scores.

\section{Extension of Theory to More than Two Domains}
\label{appendix:multidomaintheory}
We consider two or more domains $\{\mathcal{D}_1, \cdots, \mathcal{D}_N:N\geq2\}$ in which labeled data and
unlabeled data are observed.
On each domain $\mathcal{D}_i(i=1,\cdots,N)$, we observe labeled data sampled from $\mathcal{P}_i^t$ with a proportion $\alpha_i\in[0,1]$, and then unlabeled data from the marginal distribution ${\mathcal{P}_i^t}_{\mathcal{Z}}$, and by using a function $f_i:\mathcal{Z}\rightarrow\mathcal{Y}$, such as a neural network classifier, we obtain an estimate of the true distribution $\mathcal{P}_i^{f_i}=(z, f_{i}(z))_{z\sim {\mathcal{P}_i^t}_{\mathcal{Z}}}$. Therefore the distribution that we observe samples from is a mixture $\mathcal{P}_i=\alpha_i \mathcal{P}_i^t + (1-\alpha_i) \mathcal{P}_i^{f_i}$. Note that the classifiers $f_i$ need not be identical, and the proportion $\alpha_i$ on each domain $D_i$ may be also different. In addition, the overall distribution integrating all domains $\mathcal{P} = \sum_{i=1}^N q_i \mathcal{P}_i$ is obtained by mixture weights $q_i$ satisfying $\sum_{i=1}^N q_i = 1$. Note that the overall true distribution on all domains are described as $\mathcal{P}^t = \sum_{i=1}^N \frac{q_i \alpha_i}{\sum q_i \alpha_i} \mathcal{P}_i^t$, and by using $f_i$, we obtain an estimate of the true distribution on all domains $P^f=\sum_{i=1}^N \frac{q_i (1-\alpha_i)}{\sum q_i (1-\alpha_i)} \mathcal{P}_i^{f_i}$, and then $\mathcal{P} = (\sum_{i=1}^N q_i \alpha_i) \mathcal{P}^t + (\sum_{i=1}^N q_i (1-\alpha_i)) \mathcal{P}^f$.

\begin{lemma}
\label{ieq1:wasserstein_true}
Assuming that for all $i \in \{1,\cdots,N\}$, $\mathcal{P}_i^t$ and $P^t$ are densities , then we obtain
\begin{equation}
    \begin{split}
        W_1(\mathcal{P}_i^t, \mathcal{P}^t)
        &\leq \sum_{j=1}^N \frac{q_j \alpha_j}{\sum_{k=1}^N q_k \alpha_k}\left((1-\alpha_i)W_1(\mathcal{P}_i^t, \mathcal{P}_i^f) + W_1(\mathcal{P}_i, \mathcal{P}_j) + (1-\alpha_j)W_1(\mathcal{P}_j^t, \mathcal{P}_j^f)\right).
    \end{split}
\end{equation}
\end{lemma}
\begin{proof}
\begin{equation}
    \begin{split}
        W_1(\mathcal{P}_i^t, \mathcal{P}^t)
        &= W_1(\mathcal{P}_i^t, \sum_{j=1}^N \frac{q_j \alpha_j}{\sum_{k=1}^N q_k \alpha_k} \mathcal{P}_j^t) \\
        &\leq \sum_{j=1}^N \frac{q_j \alpha_j}{\sum_{k=1}^N q_k \alpha_k}W_1(\mathcal{P}_i^t, \mathcal{P}_j^t)\\
        &\leq \sum_{j=1}^N \frac{q_j \alpha_j}{\sum_{k=1}^N q_k \alpha_k}\left((1-\alpha_i)W_1(\mathcal{P}_i^t, \mathcal{P}_i^f) + W_1(\mathcal{P}_i, \mathcal{P}_j) + (1-\alpha_j)W_1(\mathcal{P}_j^t, \mathcal{P}_j^f)\right).
    \end{split}
\end{equation}
\end{proof}

\begin{lemma}
\label{ieq2:wasserstein_true}
Assuming that for all $i \in \{1,\cdots,N\}$, $\mathcal{P}_i^t$ and $\mathcal{P}^t$ are densities, then we obtain
\begin{equation}
    \forall i \in \{1,\cdots,N\}, W_1(\mathcal{P}_i^t, \mathcal{P}^t) \leq (1-\alpha_i)W_1(\mathcal{P}_i^t,\mathcal{P}_i^f) + W_1(\mathcal{P}_i, \mathcal{P}) + (1-\sum_{j=1}^N q_j\alpha_j)W_1(\mathcal{P}^t, \mathcal{P}^f).
\end{equation}
\end{lemma}
\begin{proof}
Use the triangle inequality and Lemma \ref{lemma:wasserstein_mixture_decomposition}.
\end{proof}

From Lemma \ref{ieq1:wasserstein_true} and Lemma \ref{ieq2:wasserstein_true}, we minimize the Wasserstein distance $W_1(\mathcal{P}_i,\mathcal{P}_j)$ or $W_1(\mathcal{P}_i,\mathcal{P})$ instead of the Wasserstein distance $W_1(\mathcal{P}_i^t, \mathcal{P}^t)$. To calculate $W_1(\mathcal{P}_i,\mathcal{P}_j)$ in Lemma \ref{ieq1:wasserstein_true}, we need $\mathrm{C}_2^N$ discriminators. However to calculate $W_1(\mathcal{P}_i,\mathcal{P})$ in Lemma \ref{ieq2:wasserstein_true}, we need only $N$ discriminators. The following theorem is discussed assuming that Lemma \ref{ieq2:wasserstein_true} is used.

\begin{lemma}
\label{le_diff_err_morethan2}
\begin{equation}
    \forall i \in \{1,\cdots, N\},  \abs{\mathcal{R}_{\mathcal{P}^t_i}(f)-\mathcal{R}_{\mathcal{P}^t}(f)} \leq W_{1}(\mathcal{P}^t,\mathcal{P}^t_i).
\end{equation}
\end{lemma}
\begin{proof}
Use Theorem \ref{th_diff_err}.
\end{proof}

From Lemma \ref{le_diff_err_morethan2}, we obtain for all $i \in \{1,\cdots,N\}$,  $\mathcal{R}_{\mathcal{P}^t_i}(f) \leq W_1(\mathcal{P}^t, \mathcal{P}^t_i) + \mathcal{R}_{\mathcal{P}^t}(f)$ and $\mathcal{R}_{\mathcal{P}^t}(f)\leq W_1(\mathcal{P}^t, \mathcal{P}^t_i) + \mathcal{R}_{\mathcal{P}^t_i}(f) $.

\begin{theorem}
\begin{equation}
    \abs{\max_i\{\mathcal{R}_{\mathcal{P}^t_i}(f)\}-\min_i\{\mathcal{R}_{\mathcal{P}^t_i}(f)\}}\leq  2 \max_i\{W_{1}(\mathcal{P}^t,\mathcal{P}^t_i)\}.
\end{equation}
\end{theorem}
\begin{proof}
\begin{equation}
    \begin{split}
        \forall j \in \{1,\cdots, N \},\max_i\{\mathcal{R}_{\mathcal{P}^t_i}(f)\}&\leq \max_i\{\mathcal{R}_{\mathcal{P}^t}(f)+W_{1}(\mathcal{P}^t,\mathcal{P}^t_i)\}\\
        &= \mathcal{R}_{\mathcal{P}^t}(f) + \max_i\{W_{1}(\mathcal{P}^t,\mathcal{P}^t_i)\}\\
        &\leq \mathcal{R}_{\mathcal{P}^t_j}(f) + W_1(\mathcal{P}^t, \mathcal{P}^t_j) + \max_i\{W_{1}(\mathcal{P}^t,\mathcal{P}^t_i)\}\\
        &\leq \mathcal{R}_{\mathcal{P}^t_j}(f) + 2 \max_i\{W_{1}(\mathcal{P}^t,\mathcal{P}^t_i)\}.
    \end{split}
\end{equation}
Thus, we obtain 
\begin{equation}
    \abs{\max_i\{\mathcal{R}_{\mathcal{P}^t_i}(f)\}-\min_i\{\mathcal{R}_{\mathcal{P}^t_i}(f)\}}\leq  2 \max_i\{W_{1}(\mathcal{P}^t,\mathcal{P}^t_i)\}.
\end{equation}
\end{proof}

{\small
\bibliographystyle{abbrv}
\bibliography{references}
}